\documentclass[twoside]{article}

\usepackage[accepted]{aistats2025}
\usepackage[round]{natbib}

\usepackage{hyperref}
\usepackage{url}
\usepackage{amsmath}
\usepackage{amsfonts}
\usepackage{amsthm}
\usepackage{bm}
\usepackage{subfiles}
\usepackage{graphicx}

\newtheorem{theorem}{Theorem}[section]

\newtheorem{lemma}[theorem]{Lemma}

\newtheorem{definition}[theorem]{Definition}
\newtheorem{corollary}[theorem]{Corollary}
\newtheorem{remark}[theorem]{Remark}

\DeclareMathOperator*{\trace}{Trace}
\newcommand{\dt}{\frac{\partial}{\partial t}}
\newcommand{\dtheta}{\frac{\partial}{\partial \theta}}
\newcommand{\repam}{\frac{1}{\sqrt{M}}}
\newcommand{\loss}{\mathcal{L}}
\newcommand{\calX}{\mathcal{X}}
\newcommand{\R}{\mathbb{R}}
\newcommand{\Prob}{\mathbb{P}}
\newcommand{\E}{\mathbb{E}}
\newcommand{\N}{\mathbb{N}}

\newcommand{\fcalxall}{
    \begin{pmatrix}
            f(\calX) \\
            f(\calX_+)
        \end{pmatrix}
}
\newcommand{\fcalxallinv}{
    \begin{pmatrix}
            f(\calX_+) \\
            f(\calX)
        \end{pmatrix}
}

\begin{document}

%
\runningtitle{Infinite Width Limits of SSL}

%
\runningauthor{Fleissner, Anil, Ghoshdastidar}

\twocolumn[

\aistatstitle{Infinite Width Limits of Self Supervised Neural Networks}

\aistatsauthor{ Maximilian Fleissner \And Gautham Govind Anil \And Debarghya Ghoshdastidar }

\aistatsaddress{ Technical University of Munich \And IIT Madras \And Technical University of Munich } ]

\begin{abstract}
The NTK is a widely used tool in the theoretical analysis of deep learning, allowing us to look at supervised deep neural networks through the lenses of kernel regression. Recently, several works have investigated kernel models for self-supervised learning, hypothesizing that these also shed light on the behavior of wide neural networks by virtue of the NTK. However, it remains an open question to what extent this connection is mathematically sound --- it is a commonly encountered misbelief that the kernel behavior of wide neural networks emerges irrespective of the loss function it is trained on. In this paper, we bridge the gap between the NTK and self-supervised learning, focusing on two-layer neural networks trained under the Barlow Twins loss. We prove that the NTK of Barlow Twins indeed becomes constant as the width of the network approaches infinity. Our analysis technique is a bit different from previous works on the NTK and may be of independent interest. Overall, our work provides a first justification for the use of classic kernel theory to understand self-supervised learning of wide neural networks. Building on this result, we derive generalization error bounds for kernelized Barlow Twins and connect them to neural networks of finite width.
\end{abstract}

\section{INTRODUCTION}

In recent years, self-supervised learning (SSL) has emerged as a powerful paradigm, building the foundation of several modern machine learning models. At its core, SSL relies on the idea of using augmentations to encode a notion of similarity in otherwise unlabeled data. As a typical example, consider an image dataset. Even though labels may not be known, it is reasonable to believe that randomly cropping, slightly rotating, or blurring the images will not change the true underlying class information. Therefore, two augmented versions $(x,x^+)$ of the same image should receive similar representations $f(x), f(x^+)$ in a lower-dimensional ambient space. This constitutes the basic intuition behind non-contrastive SSL, and several loss functions that capture this idea have emerged.\footnote{In contrastive SSL, one additionally incorporates a notion of dissimilarity into the model by including negative examples $x^-$ into the training procedure.} Among the most popular losses is Barlow Twins \citep{zbontar2021barlow}, a loss function that pushes the cross-correlation of the embeddings $f(x), f(x^+)$ towards the identity matrix. This aims to prevent a phenomenon known as dimension collapse, where the learned representations collapse to a single point in the embedding space.

Despite the empirical success of SSL on a range of tasks \citep{radford2021learning, bachman2019learning}, it has taken quite some time for deep learning theory to catch up with this innovation. Arguably one of the most promising avenues towards understanding the fundamental principles of SSL is by connecting it to kernel methods \citep{smola1998learning}. This of course is reminiscent of the supervised setting, where the NTK \citep{jacot2018neural, lee2019wide} provides a powerful framework to understand several phenomena of deep learning with wide neural networks, including generalization \citep{simon2021eigenlearning}, benign overfitting \citep{mallinar2022benign}, and robustness \citep{bombari2023beyond}. While a significant number of researchers are currently looking at SSL from a kernel perspective (for an overview, see related works), and while almost all of these works are motivated with the NTK, the connection is left implicit. However, it is not a priori clear that neural networks trained under SSL actually behave like kernel machines in the infinite width limit. In fact, \citet{anil2024can} have recently demonstrated that for certain contrastive loss functions, the NTK is in fact \textbf{not} constant at infinite width. This casts a shadow of doubt on the validity of kernel approximations to SSL, urging us to take a closer look at the matter.

In this paper, we bridge the gap between SSL and the NTK for the Barlow Twins loss. We prove that for neural networks with one hidden layer, the neural tangent kernel indeed becomes constant as the width of the network approaches infinity. The proof technique is different from previous works on the NTK, and leverages Grönwall's inequality (see Appendix \ref{app:grönwall}). This is necessitated by the training dynamics of the Barlow Twins loss, which make an extension of existing methods difficult. Our work confirms the hypothesized connection between kernel methods and neural network based SSL, and proves that Barlow Twins is akin to one of the most prominent representation learning methods, Kernel PCA \citep{scholkopf1997kernel}. Building on these insights, we use classic tools from learning theory to derive generalization error bounds for kernel versions of Barlow Twins, and then connect them to neural networks of finite width.

This paper is structured as follows.
We discuss related works in Section \ref{sec:related} and present our formal setup in Section \ref{sec:setup}.
Section \ref{sec:mainresult} contains our main result, stated in the general setting of multi-dimensional embeddings $f: \R^d \rightarrow \R^K$ for the Barlow twins loss. In Section \ref{sec:warmup} we give a detailed proof sketch of the one-dimensional case, which is computationally less involved but contains the main technical ideas. Section \ref{sec:implications} derives generalization error bounds for the Barlow twins loss in an abstract Hilbert space setting. Using the kernel trick and our newly established validity of the NTK approximation, we then relate these bounds to finite neural networks.
Finally, we empirically verify our findings through experiments in Section \ref{sec:exp}.

\section{RELATED WORK}\label{sec:related}

In comparison to supervised learning, the theoretical understanding of self-supervised learning is still at an early stage. Nonetheless, several works have investigated SSL using classic tools from statistical learning theory \citep{arora2019theoretical, wei2020theoretical}. Furthermore, there have been several successful attempts to look at SSL through the lenses of classic spectral and kernel methods \citep{kiani2022joint, johnson2022contrastive, haochen2021provable, cabannes2023ssl, esser2024non}. These works provide a number of useful insights into both theoretical as well as practical aspects of training SSL models, but leave the formal connection between the kernel regime and deep learning based SSL implicit. 

The idea that neural networks behave like (neural tangent) kernel models in the infinite width limit was first investigated by \citet{jacot2018neural}. Several later works further explored this connection in various contexts \citep{arora2019theoretical, chizat2019lazy, lee2019wide, liu2020linearity}, mostly for the squared (or hinge) loss. In particular, \cite{liu2020toward} develop a general framework for looking at convergence to the NTK using the Hessian. The Barlow Twins loss does not fall under the umbrella of previous analysis, since it is a fourth order loss with very different training dynamics. 

Thus, even though \citet{simon2023stepwise} motivate their investigation of kernel versions of Barlow Twins precisely with the NTK, we are not aware of any derivation that proves this analogy is valid. \citet{ziyin2022shapes} investigate the landscape of several SSL losses, but their theory is stated in the linear setting. Closest to our work is \citet{anil2024can}, who manage to bound the evolution of the NTK for certain contrastive loss functions, but do not quite prove constancy until convergence of the loss: Their bounds hold until a time that grows with network width, leaving the possibility that as wider networks are trained, the time till convergence also grows.

\section{FORMAL SETUP}\label{sec:setup}

In this work, we analyze an idealized version of the popular Barlow Twins loss \citep{zbontar2021barlow}, as considered in previous theoretical works on SSL \citep{simon2023stepwise}. Training data consists of positive pairs, denoted $\calX = \{x_1, \dots, x_N\} \subset \R^d$ and $\calX_+ = \{x^+_1, \dots, x^+_N\} \subset \R^d$. We assume data lies on the unit sphere, that is $\|x_n\| = \|x_n^+\| = 1$ for all $n \in [N]$. For example, each pair $(x_n, x_n^+)$ could consist of two (normalized) augmentations of the same underlying image, randomly cropped or blurred. The goal of Barlow Twins is now to learn the parameters of a neural network $f: \R^d \rightarrow \R^K$ that embeds each point into a lower-dimensional space. Typically, $K \ll d$. To encourage useful representations, the Barlow Twins loss function pushes the cross-moment matrix of $f(\calX)$ and $f(\calX_+)$ to the identity matrix in $\R^K$. Defining
\begin{align} \begin{split}
    C = \frac{1}{2N}  \sum_{n=1}^N f(x_n) f(x_n^+)^\top + f(x_n^+) f(x_n)^\top
\end{split} \end{align}
we minimize the loss function $\loss(f) = \left \| C - I  \right \|_F^2$ over the parameters of a neural network $f$. In this work, we restrict our analysis to two-layer neural networks, that is
\begin{align}
    f(x) = \repam \sum_{m=1}^M w_m \phi(v_m^\top x)
\end{align}
where $w_m \in \R^K$ and $v_m \in \R^d$ for all $m \in [M]$, and $\phi$ is a bounded, smooth activation function, and has bounded first derivative. For example, we could have $\phi(t) = \tanh(t)$. We denote $c_\phi$ and $c_{\phi'}$ for the supremum norms of $\phi$ and its derivative. We will later also discuss the ReLU activation. The weights are initialized as random independent Gaussians with constant variance, and collected in a vector $\theta \in \R^{M(d+K)}$ that is trained under gradient flow
\begin{align}
    \frac{\partial \theta}{\partial t} = \dot{\theta}(t) = - \frac{\partial \loss}{\partial \theta}
\end{align}
We write $\theta_0$ for the weights at initialization, and sometimes denote $f(x;\theta)$ to emphasize that $f$ is a function both of the inputs as well as of its parameters.
The neural tangent kernel is defined as a time-varying, matrix-valued map $K_t : \R^{d} \times \R^{d} \rightarrow \R^{K \times K}$, where 
\begin{align}
    K_t (x,x') = \left( \left( \frac{\partial f_k(x)}{\partial \theta(t)} \right)^\top \left( \frac{\partial f_l(x')}{\partial \theta(t)} \right) \right)_{k,l=1}^K
\end{align}
for all $x,x' \in \R^d$, and $f_k$ is the $k$-th output dimension of $f$. To underline the dependence of the NTK on the parameters $\theta(t)$ that evolve during training, we sometimes also denote it as $K_{\theta}$. The key insight of the NTK literature is that the NTK does not change during training if the width of the neural network approaches infinity. Consequently, the training dynamics of $f$ approach those of kernel regression with respect to the (vector-valued) kernel at initialization $K_0$. 

The constancy of the NTK in the infinite width limit essentially relies on three facts: Firstly, the spectral norm of the Hessian of the neural network is $\mathcal{O}(\frac{R}{\sqrt{M}})$ for all weights $\theta$ with $\| \theta - \theta_0 \| \le R$. Secondly, the change in the NTK from $\theta_0$ to any $\theta$ can be bounded in terms of the Hessian and $\| \theta - \theta_0 \|$. And thirdly, $R$ is independent of $M$ because convergence happens in a ball of width-independent radius around $\theta_0$. The first fact is true regardless of the loss function, see Lemma 7 in \citep{anil2024can}. The same is true for the second fact: We recall the following Lemma 8 by \citet{anil2024can} which extends the result by \citet{liu2020linearity} to functions $f$ with multivariate outputs.

\begin{theorem}\label{theo1}
    Consider an initial parameter $\theta_0 \in \R^p$ and a ball $B(\theta_0,R)$ around $\theta_0$ of radius $R>0$. Suppose that $\forall k \in [K]$, all inputs $a$ and all $\theta \in B(\theta_0,R)$, the Hessian of $f_k(a;\theta)$ in parameter space satisfies
    \begin{align}
        \left \| \nabla_\theta^2 f_k(a;\theta) \right\|_2 \le \epsilon
    \end{align}
    and that $\| \nabla_\theta f_k(a;\theta) \|_2 \le c_0$. Then, the change in the neural tangent kernel is bounded by 
    \begin{align} \begin{split}
        \left| \left(K_{\theta_0}(a,b)\right)_{k,l} - \left( K_{\theta}(a,b)\right)_{k,l} \right| \le 2 \epsilon c_0 R
    \end{split} 
    \end{align}
    for all $\theta \in B(\theta_0,R)$, all $k,l \in [K]$ and all inputs $a,b$. Moreover, for the neural networks considered in this paper, $c_0 = \mathcal{O}(1)$ whenever $R< \sqrt{M}$, see Appendix \ref{app:boundc0}.
\end{theorem}

However, the third piece in the puzzle is missing: Unless $R$ stays independent of $M$, we do not obtain constancy of the NTK at large width $M \rightarrow \infty$. For the squared error, \citet{liu2020toward} use the Polyak-Lojasiewicz condition to ensure that the weights remain in a bounded ball of width-independent radius $R$. It is not obvious how to extend this idea to Barlow Twins loss. In this paper, we therefore take a slightly different approach, and look at the evolution of the loss and the parameters \textit{in time}.

\section{MAIN RESULT}\label{sec:mainresult}

In light of the aforementioned discussion, we focus solely on proving that the weights of the network stay in a ball of fixed radius $R$, with high probability over random initialization. To this end, we fix a global constant $\delta>0$, below which the loss is considered to be zero, and training is stopped. Our strategy is to verify the following two statements, both of which hold at large width $M$.

\begin{enumerate}
    \item For any finite time $T>0$, there exists a width-independent $\kappa>0$ such that $\sup_{t \le T} \| \dot{\theta}(t) \| \le \kappa$ with high probability for any network of sufficiently large width.
    \item There exists a finite time $T$ such that $\loss(T) \le \delta$ for any network of sufficiently large width.
\end{enumerate}

Together, both imply that for large enough $M$, the weights $\theta$ remain within a ball of width-independent radius $R$ around the initial $\theta_0$ until $T$, with high probability. The reason is the following: Writing $u(t) = \| \theta_t - \theta_0 \|^2$, Cauchy-Schwartz implies
\begin{align}
\begin{split}
    \dt u(t)
    &= 2 \sum_{m=1}^M \left( \theta_m(t) - \theta_m(0) \right) \dot{\theta_m}(t) \\
    &\le 2 \| \theta_t - \theta_0 \| \cdot \| \dot{\theta}(t) \| \\
    &\le 2 \kappa \sqrt{u(t)}
\end{split}
\end{align}
Since $u(0)=0$, the comparison principle for ordinary differential equations shows $u(t) \le \kappa^2 t^2$ for all $t \le T$. Thus, the weights $\theta$ remain in a ball of radius $R = \kappa T$ around $\theta_0$ until convergence. We emphasize that $\kappa$ can (and will) depend on $T$ as well.

\begin{theorem}\label{theo:multi_gradientbound} \textbf{(Gradient of weights is bounded)}
    Fix any $T \ge 0$ and any $\epsilon>0$. Then, there exists $M_0 \in \N$ and some $\kappa>0$ such that, for all networks of width $M>M_0$, with probability at least $1-\epsilon$, the weights $\theta$ satisfy $\| \dot{\theta}(t) \| \le \kappa$ for all $t \le T$ when trained under gradient flow.
\end{theorem}

The proof is included in Appendix \ref{app:theo:multi_gradientbound}. The statement is probabilistic because it only holds if the weights are not too large at initialization (which is certainly true with high probability). It remains to show that for some sufficiently large width $M$, the convergence of the loss indeed happens within some finite time $T$, up to our fixed tolerance $\delta > 0$. We therefore turn to the evolution of $\loss(t)$ over time. Defining
\begin{align}
\begin{split}
    u(t) &= 
    \begin{bmatrix}
        \left( \frac{\partial \loss}{\partial f_k(x_n)} \right)_{n,k} \\
        \left( \frac{\partial \loss}{\partial f_k(x_n^+)} \right)_{n,k}
    \end{bmatrix} \in \R^{2NK} \\
    \bm K(t) &= 
    \begin{bmatrix}
     K_t(x_1, x_1) & \dots & K_t(x_1, x_N^+) \\
     \dots & \dots & \dots \\
     K_t(x_N^+, x_1) & \dots & K_t(x_N^+, x_N^+) \\
    \end{bmatrix} \in \R^{2NK \times 2NK}
\end{split}
\end{align}

we express the time evolution of the Barlow Twins loss in a more concise manner.

\begin{lemma}\label{lemma_multidim_1} \textbf{(Time evolution of the loss)}
    Under gradient flow, the evolution of $\loss(t)$ can be expressed as
    \begin{align}
        \dt \loss(t) = - u(t)^\top \bm K(t) u(t)
    \end{align}
\end{lemma}

The proof is included in Appendix \ref{app:lemma_multidim_1}. Our main theorem requires an assumption on the loss at initialization.

\begin{definition}\label{def:multi_eta} \textbf{(Definitions and Assumptions at initialization)}
We assume that, under the initialization that ensures Theorem \ref{theo:multi_gradientbound} to hold, the matrix $\bm K(0)$ is positive definite, with smallest eigenvalue $\lambda_{min}(\bm K(0)) \ge \lambda > 0$, and that $\loss(0) \le 1-\rho < 1$ for some small $\rho>0$. Define
\begin{align}\label{eq:def_multi_eta}
    \eta = \frac{4\lambda (1 - \sqrt{1-\rho})}{N}
\end{align}
and choose $T = T(\eta)$ such that the solution to the autonomous ODE $\dt L = - \eta L$ satisfies $L(T) \le \delta$, under the initial condition $L(0) = 1-\rho$.
\end{definition}

In particular, $\eta$ and $T = T(\eta)$ are independent of the network width. We state the main result.

\begin{theorem}\label{theo:multiv_lossbound} \textbf{(Convergence of the loss in finite time)}
    Under the setting in Definition \ref{def:multi_eta}, there exists $M_1 \in \N$ such that, for all networks of width $M \ge M_1$, the loss under gradient flow satisfies $\dt \loss(t) < -\eta \loss(t)$ for all time $t \le T$. This implies that $\loss(T) < \delta$.
\end{theorem}

The proof of Theorem \ref{theo:multiv_lossbound} is included in Appendix \ref{app:theo:multiv_lossbound}. As a direct consequence of this result and of Theorem \ref{theo1}, we obtain the following corollary.

\begin{corollary}\label{cor:ntk} \textbf{(Convergence of the NTK)}
    Under the conditions of Theorem \ref{theo:multiv_lossbound}, there exists a radius $R>0$ such that, with probability at least $1-\epsilon$, the change of the NTK until convergence is $\mathcal{O}(R^2/\sqrt{M})$, with $R$ depending only on $\delta, \eta$, but independent of the network width.
\end{corollary}

For the ReLU activation $\phi(x) = \max(0, x)$, gradient flow itself is ill-defined, due to the non-differentiability of the ReLU at zero. However, if we define the weak derivative $\frac{\partial \phi(0)}{\partial x} = 0$, and equate $\dot{\theta}(t) = - \frac{\partial \loss}{\partial \theta}$ as per usual, then it is possible to prove Theorem \ref{theo:multi_gradientbound} nonetheless, and all other results remain true as well. Of course, this is not entirely rigorous, because $\dot{\theta}(t) = - \frac{\partial \loss}{\partial \theta}$ is not actually gradient flow. See Appendix \ref{app:relu} for details.

\begin{remark}\textbf{(Dropping the assumption on the loss)}
The condition $\loss(0)<1$ is most likely not necessary. In Appendix \ref{app:dynamics}, we prove that in the linearized regime, the loss converges exponentially to zero if all eigenvalues of the embedding cross-moment matrix $C(0)$ are contained in $(0,1)$. This suggests that ``small'' positive definite initialization is sufficient to enter the kernel regime at large width. Our experiments also do not impose it, and yet show convergence to the NTK (see Section \ref{sec:exp}). For wide ReLU networks, $\loss(0)<1$ can however be guaranteed by suitably scaling the Gaussian weights $v_m \in \R^d$ of the first layer by a data-dependent constant (see Appendix \ref{app:scaling_v}).
\end{remark}

\section{PROOF SKETCH FOR ONE DIMENSIONAL EMBEDDINGS}\label{sec:warmup}

To give a better intuition on our proof strategy, we present a more detailed proof sketch for the simplest possible case of one-dimensional embeddings $f: \R^d \rightarrow \R$. We consider a neural network with one hidden layer and $d$-dimensional inputs $x$. For one-dimensional embeddings, the Barlow Twins loss function equates to
\begin{align} \begin{split}
    \loss(f) = \Bigg( \bigg(\underbrace{\frac{1}{N} \sum_{n=1}^N f(x_n) f(x_n^+)}_{=: C} \bigg) - 1 \Bigg)^2
\end{split} \end{align}
For ease of notation we define the cross-moment matrix of the embeddings
\begin{align} \begin{split}
    C = \frac{1}{N} \sum_{n=1}^N f(x_n) f(x_n^+)
\end{split} \end{align}
As discussed in Section \ref{sec:mainresult}, we aim to verify that at sufficiently large width (i) there exists a width-independent $\kappa>0$ such that $\sup_{t \le T} \| \dot{\theta}(t) \| \le \kappa$ with high probability over random initialization of weights and (ii) $\loss(T) \le \delta$ at a finite time $T$ independent of $M$. We begin by checking the first condition.

\begin{theorem}\label{theo:univ_gradientbound} \textbf{(Gradient of weights is bounded)}
    Fix any $T \ge 0$ and any $\epsilon>0$. Then, there exists $M_0 \in \N$ and some $\kappa>0$ such that, for all networks of width $M>M_0$, with probability at least $1-\epsilon$, the weights $\theta$ satisfy $\| \dot{\theta}(t) \| \le \kappa$ for all $t \le T$ when trained under gradient flow.
\end{theorem}

The proof is included in Appendix \ref{app:theo:univ_gradientbound}. For simplicity, the proof assumes $\loss(0)<1$, although this is not necessary here yet.
Essentially, the proof proceeds in three steps. 

\begin{enumerate}
    \item Firstly, we show that $\dt \left( \sum_{m=1}^M w_m^2(t) \right) \le 8$ for all $t \le T$. This implies that there exists some $\kappa_1 > 0$ such that for all $t \le T$, we have $\frac{1}{M} \left( \sum_{m=1}^M w_m^2(t) \right) \le \kappa_1$ with probability $\ge 1- \epsilon$.
     \item From there, we bound the maximum squared value that any representation takes until time $T$, that is
    \begin{align}
        |f|^2 = \max_{n \in [N]} \sup_{t \le T} \max \left( |f(x_n)|^2, |f(x_n^+)|^2 \right)
    \end{align}
    by some $\kappa_2 > 0$ (again independent of $M$). This requires the first part of the proof, as well as Grönwall's inequality (see Appendix \ref{app:grönwall}). Grönwall's inequality introduces potentially exponential dependencies of $|f|^2$ on $T$. Note that it is not a priori clear that $|f|^2$ remains bounded: While we know that the products $f(x_n) f(x_n^+)$ cannot explode (otherwise, the loss would also explode, which is impossible under gradient flow), the individual terms could be large.
    \item Finally, we show that
    \begin{align}
        \| \dot{ \theta}(t) \|^2 \le \frac{16|f|^2}{M} \left( c_{\phi}^2 + d c_{\phi'}^2 \left( \sum_{m=1}^M w_m^2 \right) \right)
    \end{align}
    for all $t \le T$. Combining this with the first and second part we conclude that there indeed exists $\kappa>0$ such that
    \begin{align}
        \sup_{t \le T} \| \dot{ \theta}(t) \| \le \kappa
    \end{align}
    with probability $\ge 1-\epsilon$.
\end{enumerate}

With Theorem \ref{theo:univ_gradientbound} established, it remains to prove convergence of the loss in finite time, for wide neural networks. To this end, we derive the evolution of $\loss(t)$.

\begin{lemma}\label{lemma:loss_dynamics_1dim} \textbf{(Time evolution of the loss)}
    The loss $\loss(t)$ evolves over time as
    \begin{align}
        \dt \loss(t) = -\frac{4 \loss(t)}{N^2} \cdot \fcalxallinv^\top \bm{K}(t) \fcalxallinv
    \end{align}
    where we define
    \begin{align} \begin{split}
        \bm{K}(t) = 
        \begin{pmatrix}
        K_{\theta}(\calX, \calX) & K_{\theta}(\calX, \calX_+) \\
        K_{\theta}(\calX_+, \calX) & K_{\theta}(\calX_+, \calX_+) 
        \end{pmatrix}
    \end{split} \end{align}
    which is just the kernel matrix of $(\calX_+, \calX)$ at $\theta(t)$.
\end{lemma}

The proof is included in Appendix \ref{app:lemma_loss_dynamics_1dim}. We recognize this as a non-autonomous ODE, linear in $\loss$ but multiplied with a time-dependent scalar function
\begin{align} \begin{split}
    g(t) = \frac{4}{N^2}\fcalxallinv^\top \bm{K}(t) \fcalxallinv
\end{split} \end{align}
on the right side. Suppose we could find a constant $\eta>0$ such that $\forall t \le T: g(t) > \eta$. Then, the solution to the non-autonomous ODE for $\loss(t)$ can be upper bounded by the solution to the autonomous ODE
\begin{align} \begin{split}\label{eq:autonomode_bt}
    &\dot L = - \eta L \\ &L(0) = 1 - \rho
\end{split} \end{align}
because $\loss(0) < L(0)$ and $\dot \loss < \dot L < 0$. Since $L(t)$ certainly converges to zero up to our previously fixed $\delta>0$ up to $T$ as defined in Definition \ref{def:multi_eta}, so does $\loss(t)$.

\begin{theorem}\label{theo:univ_lossbound} \textbf{(Convergence of the loss in finite time)}
    Under the setting in Definition \ref{def:multi_eta}, there exists $M_1 \in \N$ such that, for all networks of width $M \ge M_1$, the loss under gradient flow satisfies $\dt \loss(t) < -\eta \loss(t)$ for all $t \le T$, by virtue of $g(t) > \eta$ for all $t \le T$. This implies that $\loss(T) < \delta$.
\end{theorem}

The proof is included in Appendix \ref{app:theo:univ_lossbound}. Essentially, it proceeds as follows: Firstly, the kernel matrix $\bm K(t)$ does not change much up to time $T$, provided $M$ is large. This is due to Theorem \ref{theo:univ_gradientbound}, which guarantees that the weights remain in a bounded ball of radius $R$ around their initialization (until time $T$), and Theorem \ref{theo1}, which then bounds the change in the NTK matrix as $\mathcal{O}(\frac{R^2}{\sqrt{M}})$. Therefore, the smallest eigenvalue of $\bm K(t)$ stays larger than $\lambda/2$ when $M$ is sufficiently large. Thus, for all $t \le T$, it holds that
\begin{align}
\begin{split}
    g(t) &> \frac{2 \lambda}{N^2} \left \| \fcalxall \right \|^2 \\
    &\ge \frac{2 \lambda}{N^2} \left( \fcalxall \right)^\top \left( \fcalxallinv \right) \\
    &= \frac{4 C(t) \lambda}{N}
\end{split}
\end{align}
where we used Cauchy-Schwartz inequality in the second line, and plugged in
\begin{align}
    C(t) = \frac{1}{2N} \fcalxall^\top \fcalxallinv
\end{align}
in the third line.
Finally, gradient flow ensures that $C(t) \ge 1-\sqrt{\loss(0)} \ge 1 - \sqrt{1-\rho}$ for all $t$, giving the desired lower bound on $g(t)$.

\section{IMPLICATIONS FOR THE GENERALIZATION ERROR OF BARLOW TWINS}\label{sec:implications}

Having established the validity of the NTK approximation for neural networks trained under the Barlow Twins loss, we leverage our newly found connection into generalization error bounds for (finite) neural networks. We wish to bound $\bar{\loss}(f) = \| \bar{C}(f) - I \|_F^2$ where the population cross-moment matrix is
\begin{align}\label{eq:poploss_bt_nn}
    \bar{C}(f) = \E_{x,x^+} \left[ \frac{f(x) f(x^+)^\top + f(x^+) f(x)^\top}{2} \right]
\end{align}
Here, $f$ denotes a neural network with $M$ neurons, and $p$ weights collected in a vector $\theta \in \R^p$. Whenever $M$ is large enough to ensure Corollary \ref{cor:ntk} kicks in, we know that with probability at least $1-\epsilon$
\begin{align}
\begin{split}
        f_k(x ; \theta_T) = &f_k(x; \theta_0) + \langle \theta_T - \theta_0, \nabla_\theta f_k(x ;\theta_0) \rangle \\ &+ \zeta
\end{split}
\end{align}
which is simply the first-order Taylor approximation of $f_k$ in parameter space. Note that $\zeta = \mathcal{O}(\frac{R^3}{\sqrt{M}})$ because $\| \theta_T - \theta_0 \| \le R$ and the Hessian scales as $\mathcal{O}(\frac{R}{\sqrt{M}})$. Therefore, the trained neural network $f_k(x; \theta_T)$ can be approximated by a kernel model
\begin{align}\label{eq:ntk_model}
\begin{split}
g_k(x) &= 
    \begin{pmatrix}
        1 \\
        \theta_T - \theta_0
    \end{pmatrix}^\top 
    \begin{pmatrix}
        f_k(x;\theta_0) \\
        \nabla_\theta f_k (x ; \theta_0) \rangle
    \end{pmatrix} \\
    &=: \left( \theta_T' - \theta_0 \right)^\top \psi_k(x)
\end{split}
\end{align}
up to some small positive $\zeta$. In this section, we therefore derive generalization error bounds for kernel versions of Barlow Twins. 

\paragraph{Generalization error for Barlow Twins in Hilbert spaces.}

We first place ourselves in an abstract Hilbert space framework. Denoting $z, z^+$ for positive pairs residing in some Hilbert space $\mathcal{H}$, we wish to learn a bounded linear operator $W: \mathcal{H} \rightarrow \R^K$ such as to minimize
\begin{align}\label{eq:poploss_bt_ntk}
\begin{split}
    &\bar{\loss}(W) = \left \| W \bar{ \Gamma} W^* - I_K \right \|_F^2 \text{ , where} \\
    &\bar{ \Gamma} = \frac{1}{2} \E_{z,z^+}\left[ z (z^+)^* + z^+ z^* \right]
\end{split}
\end{align}
Here $*$ denotes the adjoint, $\| \cdot \|_{\text{HS}}$ is the Hilbert-Schmidt norm of linear operators on $\mathcal{H}$ (in the finite-dimensional setting, the Frobenius norm $\| \cdot \|_F$), and $\| \cdot \|$ is the operator norm.

\begin{theorem}\label{theo:learningtheory} \textbf{(Generalization error for Barlow Twins in Hilbert spaces)}
        Let $N \in \N$ and $N' \le N$. Assume that pairs $(z,z^+)$ are almost surely contained in a ball of radius $S>0$ in a Hilbert space $\mathcal{H}$. Moreover, assume that for any set of training data, $\loss(W) \le \delta$ for some $\|W\| \le B$. Then, with probability $\ge 1 - 2\epsilon$, it holds that
        \begin{align}
            \begin{split}
            \bar{\loss}(W) \le 3 \delta &+ \frac{3B^4}{N} \left( \bm{ \hat V} + \exp \left( - \frac{(N'-1)^2 \epsilon^2}{8S^8 N'} \right) \right) \\
            &+ 3\exp \left( - \frac{N\epsilon^2}{B^4 S^4} \right) 
            \end{split}
        \end{align}
        where we define
        \begin{align}
            \begin{split}
                &\bm {\hat V} = \frac{1}{N'(N'-1)} \sum_{\substack{i < j \\  i,j \in [N']}} \left \| \Gamma_i - \Gamma_j \right\|^2_{\text{HS}} \\
                &\Gamma_i = \frac{1}{2} \left( z_i (z_i^+)^* + z_i^+ z_i^* \right)
            \end{split}
        \end{align}
        and the randomness is over independently sampled positive pairs $(z_1, z_1^+), \dots, (z_N, z_N^+)$.
\end{theorem}

The proof relies on McDiarmid's inequality \citep{mcdiarmid1989method} and is included in Appendix \ref{app:theo:learningtheory}. The quantity $\bm{ \hat V}$ is purely empirical, and may be estimated from fewer samples by choosing $N'<N$. Of course, estimation is a strong word here --- in an abstract Hilbert space setting, we may not be able to compute $\| \Gamma_i - \Gamma_j \|^2_{\text{HS}}$. However, when the inner product in $\mathcal{H}$ takes the form of a kernel, $\bm{ \hat V}$ can be expressed more explicitly. 

\begin{figure*}
\centering
\includegraphics[width = 0.98\textwidth]{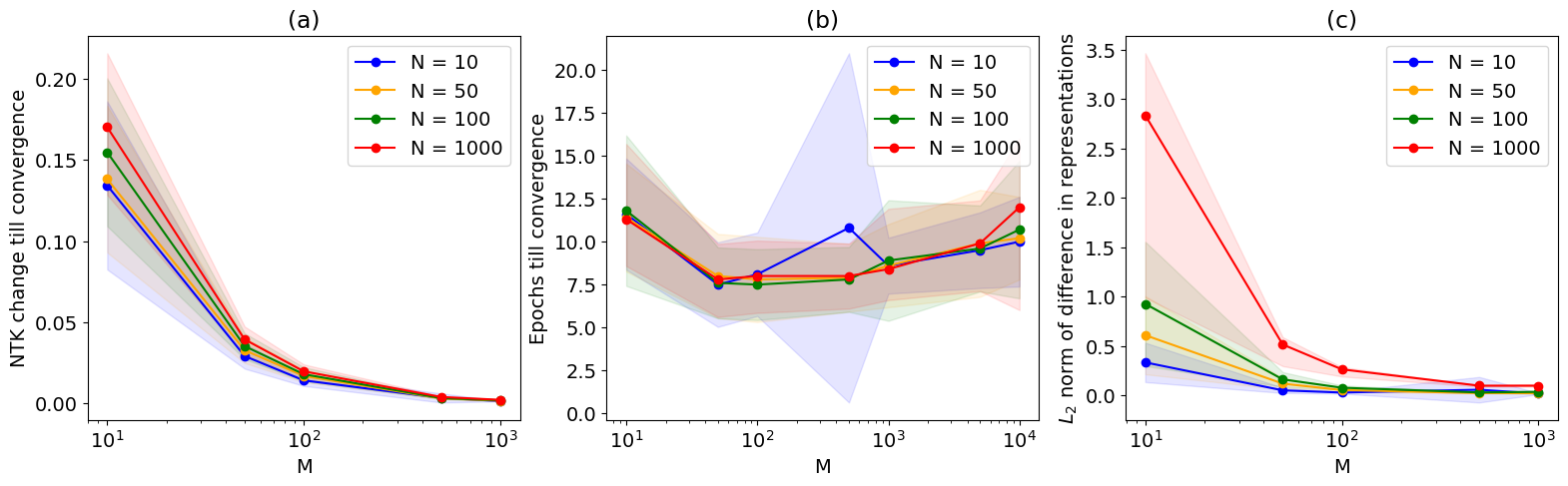}
\caption{For a fixed sample size $N$, we plot different quantities for varying network width $M$. We then vary $N$ and plot: (a) NTK change till convergence (b) Training Epochs till convergence (c) Squared norm of difference between representations of neural network and corresponding kernel model }
\label{fig:common_figure}
\end{figure*}

\paragraph{Connecting to neural networks.}
Such is the case for the NTK approximation of the neural network \eqref{eq:ntk_model}. Defining $\psi(x) = \left( \psi_1(x), \dots \psi_K(x) \right)$ as the concatenation of all feature maps, we obtain the kernel inner product
\begin{align}
    \hat{K}_{\theta_0}(x,x') = \langle \psi(x), \psi(x') \rangle = \sum_{k=1}^K \psi_k(x) \psi_k(x')
\end{align}
Then, defining
\begin{align}
    W = 
    \begin{pmatrix}
        (\theta'_T - \theta'_0)^\top & 0 & \dots \\
        0 & (\theta'_T - \theta'_0)^\top & \dots \\
        \dots & \dots & \dots
    \end{pmatrix}
    \in \R^{K \times K(p+1)}
\end{align}
we obtain $g(x) = W \psi(x)$. Actually, multivariate functions such as $g$ reside in a vector-valued RKHS. This is also clear from the fact that the NTK in the multivariate setting is a matrix-valued kernel $K_{\theta_0}$. In a vector-valued RKHS, we do not have feature maps, but rather feature operators $\psi$. For the neural network, the feature operator is the Jacobian. Under the trace inner product $\langle \psi(x), \psi(x') \rangle := \trace( \psi(x)^* \psi(x')) = \trace ( K_{\theta_0}(x,x'))$ we see that this construction is identical to the one given above, in that it gives the same inner product between $\psi(x), \psi(x')$. By virtue of the kernel trick, we can now estimate $\bm{ \hat V}$ without moving into the Hilbert space.

\begin{lemma}\label{lemma:hatv_kernel}\textbf{(Estimating $\bm{ \hat V}$ with kernels)}
    Consider the setting of Theorem \ref{theo:learningtheory}. Suppose that there exists a feature map $\psi: \R^d \rightarrow \mathcal{H}$ such that $z_n = \psi(x_n)$ and $ z_n^+ = \psi(x_n^+)$. Moreover, assume there exists a kernel $K_{\theta_0}: \R^d \times \R^d \rightarrow \R$ with $\langle \psi(x), \psi(x') \rangle = K_{\theta_0}(x,x')$
    for all $x,x' \in \R^d$. Then, for all $i \neq j$, it holds that
    \begin{align}
    \begin{split}
        &\left \| \Gamma_i - \Gamma_j \right\|^2_{\text{HS}} = \\
        &0.5 \left( K_{\theta_0}(x_i,x_i) K_{\theta_0}(x_i^+, x_i^+) + K_{\theta_0}(x_i, x_i^+)^2 \right) + \\
        &0.5 \left( K_{\theta_0}(x_j,x_j) K_{\theta_0}(x_j^+, x_j^+) + K_{\theta_0}(x_j, x_j^+)^2 \right) - \\
        &K_{\theta_0}(x_i, x_j) K_{\theta_0}(x_i^+, x_j^+) - K_{\theta_0}(x_i, x_j^+) K_{\theta_0}(x_i^+, x_j)
    \end{split}
    \end{align}
\end{lemma}

This result is proven in Appendix \ref{app:lemma:hatv_kernel}.
Furthermore, combining Theorem \ref{theo:learningtheory} with the fact that the neural network $f$ can be approximated by a kernel model \eqref{eq:ntk_model}, we immediately obtain the following corollary.

\begin{corollary}\label{cor:nn_lossbound}
    Fix $\epsilon>0$. Assume $f$ is a neural with $M$ neurons in the hidden layer, where $M$ is large enough for Corollary \ref{cor:ntk} to hold. Suppose $f$ is trained until the Barlow Twins loss is smaller than $\delta$. Then, with probability at least $1-3\epsilon$, it holds that
    \begin{align}
        \bar{\loss}(f) \le 2 \nu(N,\epsilon,\delta) + 8K^2 \zeta^2 (2BS + \zeta)^2 
    \end{align}
    where $\nu(N, \epsilon, \delta)$ is the slack term from Theorem \ref{theo:learningtheory}. $B$ can be taken as the norm of $\|\theta_T' - \theta_0'\| \le R + 1$, and $S = \sup_{x \in \R^d} \| \psi(x) \|$.
\end{corollary}

The proof is included in Appendix \ref{app:nn_lossbound}. It uses the fact that $f$ is approximated up to $\zeta$ by a kernel model with probability $\ge 1-\epsilon$. This allows bounding the difference between the loss of $f$ and the loss of $g$. Then, we use Theorem \ref{theo:learningtheory}, which holds with probability at least $1-2\epsilon$.

\begin{remark}
    Our analysis only bounds the pretraining loss. However, this can be passed on to guarantees on the classification error on downstream tasks, provided the augmentations elucidate enough of the underlying class structures. For example, see Section 5.2 of \citet{cabannes2023ssl}.
\end{remark}

\section{EXPERIMENTS}\label{sec:exp}

In this Section, we verify our theoretical claims on the MNIST dataset \citep{deng2012mnist}. Optimization is done using gradient descent with a learning rate of $0.5$. The threshold for loss convergence $\delta$ is set at $10^{-5}$. We use a single-hidden layer neural network with tanh activation unless stated otherwise. All experiments are run 10 times with different random seeds; their means along with standard deviations are plotted in the figures.

We first verify near-constancy of the NTK for Barlow Twins at large width, with one-dimensional embeddings. We do not restrict ourselves to the setting of $\loss(0)<1$. Varying the sample size from $10$ to $1000$, we plot the norm of the NTK deviation by varying the hidden layer width $M$ from $10$ to $10000$. Recall that in our theoretical results, $\eta$ depends on the sample size $N$, hence does $T$, hence does $R$ and hence does $M$. This suggests that as $N$ grows, a larger width is necessary for the kernel regime to kick in. However, as Figure \ref{fig:common_figure} (a) clearly displays, the neural network enters the kernel regime irrespective of the sample size $N$ as $M$ grows, with the change in the NTK near zero at $10^3$ neurons in the hidden layer.

Next, we look at the number of training epochs required for the loss to converge below $\delta$. As per Theorem \ref{theo:univ_lossbound}, we expect this to be independent of the width of the neural network, provided the network is sufficiently large. This is empirically verified in  Figure \ref{fig:common_figure} (b).

With the NTK nearly constant at large width, we expect the representations learned by the finite neural network to be close to those learned by a corresponding kernel model, via gradient descent under the NTK at initialization $\bm K(0)$. We verify this in Figure \ref{fig:common_figure} (c), where we see that irrespective of the number of samples, the representations at convergence are closer for wider neural networks
Additional experiments are included in Appendix \ref{app:exp}. There, we also check the variation of these quantities for networks with ReLU activation, and include results for higher-dimensional embeddings.

\section{DISCUSSION}

In this paper, we connect self-supervised learning with neural networks to the neural tangent kernel regime. We prove that at infinite width, the NTK of a neural network trained under the Barlow Twins loss becomes constant. This is the first result that rigorously justifies the use of traditional kernel methods to understand SSL through the lenses of the NTK. Furthermore, the kernel connection enables us to bound the population Barlow Twins loss in terms of the empirical loss and a slack term that decays with the number of samples. This opens a number of interesting avenues and challenges for future work. 

\paragraph{Extension to deep networks, and other losses.} It is desirable to extend the convergence results to other, in particular deep, architectures. This requires an extension of Theorem \ref{theo:multi_gradientbound} to multiple layers. In addition, verifying the validity of the NTK approximation for other commonly used non-contrastive loss functions such as VIC-Reg \citep{bardes2021vicreg} is a natural next step. Moreover, our proof for the ReLU activations is not entirely rigorous, since $\frac{\partial \loss}{\partial \theta}$ can only be understood in the weak sense. Finally, as discussed earlier, we believe the condition $\loss(0)<1$ can be dropped, but a proof remains to be established.

\paragraph{Improving generalization error bounds.} Theorem \ref{theo:learningtheory} still relies on uniform bounds, despite closed-form expressions for kernel versions of Barlow Twins being available. To be precise, \citet{simon2023stepwise} show that the minimum norm $W^*$ that achieves zero loss lies in the top eigenspace of the cross-moment operator $\Gamma$. This $W^*$ is referred to as the \textit{spectral solution}. We believe that spectral perturbation bounds provide tighter guarantees on the generalization error.

\paragraph{Does gradient flow approach the spectral solution?}
Even with improved generalization bounds for the spectral solution $W^*$ (as noted above), there is still a missing link: It is not yet known whether gradient flow actually approaches the spectral solution for linearized networks. \citet{simon2023stepwise} prove it for ``aligned'' initialization (starting in the top eigenspace), and give heuristics for why it would also hold under ``small'' initialization.

\paragraph{Conclusion.} Our work establishes the significance of the NTK to self-supervised learning, and confirms numerous recent works that have headed into this direction. As \citet{belkin2018understand} put it, \emph{to understand deep learning, we need to understand kernel learning}. In light of our results, we are tempted to add: To understand self-supervised learning, we need to understand representation learning with kernels.

\subsubsection*{Acknowledgements}

This paper is supported by the DAAD programme Konrad Zuse Schools of Excellence in Artificial Intelligence, sponsored by the Federal Ministry of Education and Research. We also thank the anonymous reviewers.


\bibliography{ssl_dynamics.bib}


\clearpage

\section*{Checklist}


\begin{enumerate}

\item For all models and algorithms presented, check if you include:
\begin{enumerate}
\item A clear description of the mathematical setting, assumptions, algorithm, and/or model. [\textbf{Yes}/No/Not Applicable]
\item An analysis of the properties and complexity (time, space, sample size) of any algorithm. [\textbf{Yes}/No/Not Applicable]
\item (Optional) Anonymized source code, with specification of all dependencies, including external libraries. [\textbf{Yes, we add a link to the code} /No/Not Applicable]
\end{enumerate}

\item For any theoretical claim, check if you include:
\begin{enumerate}
\item Statements of the full set of assumptions of all theoretical results. [\textbf{Yes}/No/Not Applicable]
\item Complete proofs of all theoretical results. [\textbf{Yes}/No/Not Applicable]
\item Clear explanations of any assumptions. [\textbf{Yes}/No/Not Applicable]     
\end{enumerate}

\item For all figures and tables that present empirical results, check if you include:
\begin{enumerate}
\item The code, data, and instructions needed to reproduce the main experimental results (either in the supplemental material or as a URL). [\textbf{Yes, we add a link to the code}/No/Not Applicable]
\item All the training details (e.g., data splits, hyperparameters, how they were chosen). [\textbf{Yes, see appendix} /No/Not Applicable]
\item A clear definition of the specific measure or statistics and error bars (e.g., with respect to the random seed after running experiments multiple times). [\textbf{Yes}/No/Not Applicable]
\item A description of the computing infrastructure used. (e.g., type of GPUs, internal cluster, or cloud provider). [\textbf{Yes}/No/Not Applicable]
\end{enumerate}

\item If you are using existing assets (e.g., code, data, models) or curating/releasing new assets, check if you include:
\begin{enumerate}
\item Citations of the creator If your work uses existing assets. [\textbf{Yes}/No/Not Applicable]
\item The license information of the assets, if applicable. [Yes/No/\textbf{Not Applicable}]
\item New assets either in the supplemental material or as a URL, if applicable. [Yes/No/\textbf{Not Applicable}]
\item Information about consent from data providers/curators. [Yes/No/\textbf{Not Applicable}]
\item Discussion of sensible content if applicable, e.g., personally identifiable information or offensive content. [Yes/No/\textbf{Not Applicable}]
\end{enumerate}

\item If you used crowdsourcing or conducted research with human subjects, check if you include:
\begin{enumerate}
\item The full text of instructions given to participants and screenshots. [Yes/No/\textbf{Not Applicable}]
\item Descriptions of potential participant risks, with links to Institutional Review Board (IRB) approvals if applicable. [Yes/No/\textbf{Not Applicable}]
\item The estimated hourly wage paid to participants and the total amount spent on participant compensation. [Yes/No/\textbf{Not Applicable}]
\end{enumerate}

\end{enumerate}


\clearpage

\appendix
\renewcommand{\thesection}{\Alph{section}}

\aistatstitle{Supplementary Materials}
\onecolumn




\section{PROOFS FROM SECTION \ref{sec:mainresult}}\label{app:mainresult}

Appendix \ref{app:warmup} contains a more detailed and computationally less involved derivation of all results for the one-dimensional case. It may be convenient to refer to it first. The proof strategy is identical.

\subsection{Proof of Theorem \ref{theo:multi_gradientbound}}\label{app:theo:multi_gradientbound}

\begin{proof}
    Recall that $T>0$ is a fixed, width-independent point in time. We show that there exists $\kappa>0$ such that $\sup_{t \le T} \| \dot{\theta} \| \le \kappa$ for all networks of width $M>8T$, with high probability over randomly initialized weights. Specifically, for any $\epsilon>0$, we know that there exists $\kappa_1>0$ such that
    \begin{align}
        \frac{\| \theta(0)\|^2}{M} \le \kappa_1 - 1
    \end{align}
    with probability at least $1-\epsilon$, because the weights are independent Gaussians at initialization. We condition everything on this event. In this proof, we also assume $\loss(0)<1$ under said event. We need this condition anyway for Theorem \ref{theo:multiv_lossbound} and therefore include it into our bounds here as well to simplify matters. However, any width-independent bound on $\loss(0)$ is sufficient. Denote $c_\phi = \max_{t \in \R} | \phi(t)|$ and $c_{\phi'} = \max_{t \in \R} \left| \dt \phi(t) \right|$. We will show the following three statements.
    \begin{enumerate}
        \item Firstly, we show that 
        \begin{align}
            \dt \left( \sum_{m=1}^M \|w_m (t) \|^2 \right) \le 8
        \end{align}
        for all $t \le T$. This immediately implies that there exists some $\kappa_1 > 0$ such that for all $t \le T$
        \begin{align}
            \frac{1}{M} \left( \sum_{m=1}^M \| w_m(t) \|^2 \right) \le \kappa_1
        \end{align}
        holds with probability $\ge 1- \epsilon$ over random initialization, because
        \begin{align}
            \frac{1}{M} \left( \sum_{m=1}^M \| w_m(t) \|^2 \right) \le \frac{1}{M} \left( \sum_{m=1}^M \|w_m (0) \|^2 \right) + \frac{8T}{M} \le \frac{1}{M} \left( \sum_{m=1}^M \|w_m (0) \|^2 \right) + 1
        \end{align}
        where we used $M > 8T$ and the fact that
        \begin{align}
            \frac{1}{M} \left( \sum_{m=1}^M \|w_m (0) \|^2 \right) \le
            \frac{\| \theta(0)\|^2}{M} \le \kappa_1 - 1
        \end{align}
        with probability at least $1-\epsilon$.
        \item Using part 1, we bound the maximum squared value that any representation takes until time $T$, that is
        \begin{align}
            |f|^2 = \max_{n \in [N]} \sup_{t \le T} \max \left( \|f(x_n)\|^2, \|f(x_n^+)\|^2 \right)
        \end{align}
        by some $\kappa_2 > 0$ (again independent of $M$).
        \item We combine both statements to show that there exists $\kappa>0$ such that $\sup_{t \le T} \| \dot{\theta} \| \le \kappa$ with probability at least $1-\epsilon$.
    \end{enumerate}
    
    \textbf{Part 1.} First of all, we derive the evolution of all weights over time. Using the matrix chain rule,
    \begin{align}\label{eq:dwdt_multi}
        \begin{split}
            \frac{\partial \loss}{\partial w_{mk}} &= \trace \left( \left( \frac{\partial \loss}{\partial C}\right)^\top \frac{\partial C}{\partial w_{mk}} \right) \\
            &= 2 \trace \left( (C-I) \left( \frac{1}{2N} \frac{\partial}{\partial w_{mk}} \left( \sum_{n=1}^N f(x_n) f(x_n^+)^\top + f(x_n^+) f(x_n)^\top \right) \right) \right)
        \end{split} 
    \end{align}
    and similarly,
    \begin{align}\label{eq:dvdt_multi}
        \begin{split}
            \frac{\partial \loss}{\partial v_{mr}} &= \trace \left( \left( \frac{\partial \loss}{\partial C} \right)^\top \frac{\partial C}{\partial v_{mr}} \right) \\
            &= 2 \trace \left( (C-I) \left( \frac{1}{2N} \frac{\partial}{\partial v_{mr}} \left( \sum_{n=1}^N f(x_n) f(x_n^+)^\top + f(x_n^+) f(x_n)^\top \right) \right) \right)
        \end{split}
    \end{align}
    For any $n \in [N]$ and any $i,j,k \in [K]$
    \begin{align} \begin{split}
         &\frac{\partial}{\partial w_{mk}} \left( f_i(x_n) f_j(x_n^+) + f_j(x_n) f_i(x_n^+) \right) \\
         &= \begin{cases}
         2f_k(x_n^+) \repam \phi(v_m^\top x_n) + 2f_k(x_n) \repam \phi(v_m^\top x_n^+), \text{ if } k=i=j \\
         f_j(x_n^+) \repam \phi(v_m^\top x_n) + f_j(x_n) \repam \phi(v_m^\top x_n^+), \text{ if } k=i \neq j \\
         f_i(x_n^+) \repam \phi(v_m^\top x_n) + f_i(x_n) \repam \phi(v_m^\top x_n^+), \text{ if } k=j \neq i \\
         0, \text{ if } k \notin \{i,j\}
         \end{cases}
    \end{split} \end{align}
    Equating $\dt w_{mk} = -\frac{\partial \loss}{\partial w_{mk}}$ gives
    \begin{align} \begin{split}
        \dt \left( \sum_{m=1}^M \|w_m\|^2 \right) &= 2 \sum_{m=1}^M \sum_{k=1}^K w_{mk} \left( \dt w_{mk} \right) \\
        &= 4 \trace \left( (I-C) \left( \frac{1}{2N} \sum_{n=1}^N \sum_{k=1}^K A_{k,i,j}(n) \right)_{i,j=1}^K \right)
    \end{split} \end{align}
    where we define
    \begin{align}
        \begin{split}
            A_{k,i,j}(n) &= \sum_{m=1}^M w_{mk} \cdot \frac{\partial}{\partial w_{mk}} \left( f_i(x_n) f_j(x_n^+) + f_i(x_n^+) f_j(x_n) \right) \\
        &= \begin{cases}
         2f_k(x_n^+) f_{k}(x_n) + 2f_k(x_n) f_{k}(x_n^+), \text{ if } k=i=j \\
         f_j(x_n^+) f_{k}(x_n) + f_j(x_n) f_{k}(x_n^+), \text{ if } k=i \neq j \\
         f_i(x_n^+) f_{k}(x_n) + f_i(x_n) f_{k}(x_n^+), \text{ if } k=j \neq i \\
         0, \text{ if } k \notin \{i,j\}
         \end{cases}
        \end{split}
    \end{align}
    Noticing that
    \begin{align}
        \begin{split}
            \sum_{k=1}^K A_{k,i,j}(n) =
             \begin{cases}
                 2f_i(x_n^+) f_{i}(x_n) + 2f_i(x_n) f_{i}(x_n^+), \text{ if } i=j \\
                 2f_i(x_n^+) f_{j}(x_n) + 2f_j(x_n) f_{i}(x_n^+), \text{ if } i \neq j 
             \end{cases}
        \end{split}
    \end{align}
    we obtain
    \begin{align}
    \begin{split}
         \dt \left( \sum_{m=1}^M \|w_m\|^2 \right) = 8 \trace \left( (I-C)C \right)
    \end{split}
    \end{align}
    This is certainly bounded by some $\kappa_0>0$, because the loss is non-increasing under gradient flow, which forces all entries of $C$ to stay bounded. In our case, with $\loss(0) = \|I-C\|_F^2<1$, we may choose $\kappa_0 = 8$.
    
    Regarding the evolution of the first layer weights $v_{mr}$, we note that
    for any $n \in [N]$, $i,j \in [K]$ and $r \in [d]$
    \begin{align} \begin{split}
        &\frac{\partial}{\partial v_{mr}} \left( f_i(x_n) f_j(x_n^+) + f_j(x_n) f_i(x_n^+) \right) = \\
        &\repam f_j(x_n^+) w_{mj} \phi'(v_m^\top x_n) x_n^{(r)} + \\ 
        &\repam f_i(x_n) w_{mi} \phi'(v_m^\top x_n^+) \left(x_n^+ \right)^{(r)} + \\ 
        &\repam f_j(x_n) w_{mi} \phi'(v_m^\top x_n^+) \left(x_n^+ \right)^{(r)} + \\ 
        &\repam f_i(x_n^+) w_{mi} \phi'(v_m^\top x_n) x_n^{(r)}
    \end{split} \end{align}
    Thus, we obtain
    \begin{align}
    \begin{split}
         \dot{v_{mr}}(t) = \frac{1}{N} \trace \left( (C-I) \left( \frac{\partial}{\partial v_{mr}} \left( f_i(x_n) f_j(x_n^+) + f_j(x_n) f_i(x_n^+) \right) \right)_{i,j} \right) 
    \end{split}
    \end{align}

    \textbf{Part 2.} In this part we bound $|f|^2$. For any $x$ and any $k \in [K]$, the $k$-th output $f_k$ evolves over time as follows.
    \begin{align}
    \begin{split}
        \dt f_k(x) &= \langle \frac{\partial f_k(x)}{\partial \theta}, \dt \theta \rangle \\
        &= - \sum_{m=1}^M \frac{\partial f_k(x)}{\partial \theta_m} \frac{\partial \loss}{\partial \theta_m} \\
        &= - \sum_{m=1}^M \frac{\partial f_k(x)}{\partial \theta_m} \sum_{n=1}^N \sum_{l=1}^K \frac{\partial \loss}{\partial f_l(x_n)} \frac{\partial f_l(x_n)}{\partial \theta_m} + \frac{\partial \loss}{\partial f_l(x_n^+)} \frac{\partial f_l(x_n^+)}{\partial \theta_m} \\
        &= - \sum_{n=1}^N \sum_{l=1}^K \frac{\partial \loss}{\partial f_l(x_n)} \left( K_\theta(x, x_n) \right)_{k,l} + \frac{\partial \loss}{\partial f_l(x_n^+)} \left( K_\theta(x, x_n^+\right)_{k,l}
    \end{split}
    \end{align}
    So we need to derive all $\frac{\partial \loss}{\partial f_k(x_i)}$. For any $k \in [K]$, and any $i \in [N]$, we see that $\frac{\partial C}{\partial f_k(x_i)} = \frac{1}{2N} A(i,k)$, where we define the matrix
    \begin{align}
    \begin{split}
        A(i,k) =
        \begin{pmatrix}
            0 & \dots & 0 & f_1(x_i^+) & 0 & \dots & 0 \\
            0 & \dots & 0 & f_2(x_i^+) & 0 & \dots & 0 \\
            \dots & \dots & \dots & \dots & \dots & \dots & \dots \\
            f_1(x_i^+) & f_2(x_i^+) & \dots & 2f_k(x_i^+) & \dots & \dots & f_K(x_i^+) \\
            \dots & \dots & \dots & \dots & \dots & \dots & \dots \\
            0 & \dots & 0 & f_{K-1}(x_i^+) & 0 & \dots & 0 \\
            0 & \dots & 0 & f_K(x_i^+) & 0 & \dots & 0
        \end{pmatrix}
    \end{split}
    \end{align}
    and also let
    \begin{align}
    A(i) = 
    \begin{pmatrix}
        f_1(x_i^+) & \dots & f_1(x_i^+) \\
        \dots & \dots & \dots \\
        f_K(x_i^+) & \dots & f_K(x_i^+)
    \end{pmatrix} \in \R^{K \times K}
    \end{align}
    Notice that $ A(i,k) = e_k e_k^\top A(i)^\top + A(i) e_k e_k^\top$. Using the matrix chain rule, and the fact that $\frac{\partial \loss}{\partial C} = 2(C-I)$, we therefore obtain
    \begin{align}
    \begin{split}
        \frac{\partial \loss}{\partial f_k(x_i)} &= 2 \trace \left( (C-I)^\top \frac{\partial C}{\partial f_k(x_i)} \right) \\
        &= \frac{1}{N} \trace \left( (C-I) A(i,k) \right) \\
        &= \frac{1}{N} \left( \sum_{l=1}^K (C-I)_{l \bullet} A(i,k)_{\bullet l} \right) \\
        &= \frac{1}{N} \left( \sum_{l \neq k} (C-I)_{lk} f_l(x_i^+) + (C-I)_{\bullet k}^\top f(x_i^+) \right) \\
        &= \frac{1}{N} \left( \sum_{l \neq k} C_{lk} f_l(x_i^+) + (C-I)_{\bullet k}^\top f(x_i^+) \right) \\
        &= \frac{2}{N} \left( C_{1k} f_1(x_i^+) + C_{2k} f_2(x_i^+) + \dots + (C_{kk}-1) f_k(x_i^+) + \dots + C_{Kk} f_K(x_i^+) \right) \\
        &= \frac{2}{N} (C-I)_{k \bullet} f(x_i^+) \\
        &= \frac{2}{N} e_k ^\top (C-I) f(x_i^+)
    \end{split}
    \end{align}
    where we exploited the symmetry of $A(i,k)$ in the sixth step. Overall, we have
    \begin{align}
        \begin{split}
            \dt f_k(x) &= - \frac{2}{N} \sum_{n=1}^N \sum_{l=1}^K K_{\theta}(x, x_n)_{k,l} \cdot e_l ^\top (C-I) f(x_n^+) + \\
            &\qquad \qquad \qquad \quad K_{\theta}(x, x_n^+)_{k,l} \cdot e_l ^\top (C-I) f(x_n) \\
            &= -\frac{2}{N} \sum_{l=1}^K K_\theta(x, [\calX, \calX_+])_{k,l} \cdot \bm F([\calX_+, \calX])^\top (C-I) e_l
        \end{split}
    \end{align}
    where we write $\bm F([\calX_+, \calX]) \in \R^{K \times 2N}$ for the matrix that contains the representations at time $t$ as columns.
    Moreover, with $\bm K_{k,l}$ the kernel matrix of $\calX, \calX_+$ at output $k,l$, we have
    \begin{align}
        \dt f_k([\calX, \calX_+]) = \frac{2}{N} \sum_{l=1}^K \bm K_{k,l} \cdot \bm F([\calX_+, \calX])^\top (I-C) e_l 
    \end{align}
    Therefore, the squared norm of the representations evolve as
    \begin{align}
    \begin{split}
        \dt \left( \sum_{k=1}^K \left \| f_k([\calX, \calX_+]) \right\|^2 \right) &= 2 \sum_{k=1}^K \left( f_k([\calX, \calX_+]) \right)^\top \left( \dt f_k([\calX, \calX_+]) \right) \\
        &=  \frac{4}{N} \sum_{k=1}^K \left( f_k([\calX, \calX_+]) \right)^\top \sum_{l=1}^K \bm K_{k,l} \cdot \bm F([\calX_+, \calX])^\top (C-I) e_l \\
        &\le \frac{4}{N} \sum_{k,l=1}^K \left\| f_k([\calX, \calX_+]) \right\| \left\| \bm K_{k,l} \bm F([\calX_+, \calX])^\top (C-I) e_l \right \| \\
        &\le \frac{4}{N} \sum_{k,l=1}^K \left\| f_k([\calX, \calX_+]) \right\| \cdot \| \bm K_{k,l} \|_2 \cdot \left\| \bm F([\calX_+, \calX])^\top (C-I) e_l \right \| \\
        &\le \frac{4}{N} \sum_{k,l=1}^K \left\| f_k([\calX, \calX_+]) \right\| \cdot \| \bm K_{k,l} \|_2 \cdot \left\| \bm F([\calX_+, \calX]) \right \|_2
    \end{split} 
    \end{align}
    where we have used Cauchy-Schwartz inequality, and the fact that $\|I-C\|_2 \le 1$ due to $\loss(0)<1$ in the last line. Continuing,
    \begin{align}
    \begin{split}
        &\frac{4}{N} \sum_{k,l=1}^K \left\| f_k([\calX, \calX_+]) \right\| \cdot \| \bm K_{k,l} \|_2 \cdot \left\| \bm F([\calX_+, \calX]) \right \|_2
        \le \\
        &\frac{4}{N} \max_{k,l} \| \bm K_{k,l} \|_2 \cdot \left( \sum_{k=1}^K \left\| f_k([\calX, \calX_+]) \right\| \left( \sum_{l=1}^K \left\| f_l([\calX, \calX_+]) \right\|^2 \right)^{1/2} \right) \le \\
        &\frac{4}{N} \max_{k,l} \| \bm K_{k,l} \|_2 \cdot \left( \sum_{k,l=1}^K \left\| f_k([\calX, \calX_+]) \right\| \cdot \left\| f_l([\calX, \calX_+]) \right\| \right) \le \\
        &\frac{4K}{N} \max_{k,l} \| \bm K_{k,l} \|_2 \cdot \left( \sum_{k=1}^K \left\| f_k([\calX, \calX_+]) \right\|^2 \right)
    \end{split}
    \end{align}
    where in the second line we bound the spectral norm of $\bm F([\calX_+, \calX])$ by its Frobenius norm, in the third line we bound the $2$-norm by the $1$-norm, and in the fourth line we use Cauchy-Schwartz again, which introduces a factor of $K$. We proceed to bound the spectral norm of the kernel matrices $\bm K_{k,l}$.
    
    Note that for any $k,l$, and any $(a,b) \in [\calX, \calX_+] \times [\calX, \calX_+]$, we have
    \begin{align}
        \begin{split}
            K_{k,l}(a,b) &= \left( \frac{\partial f_k(a)}{\partial \theta} \right)^\top \left( \frac{\partial f_l(a)}{\partial \theta} \right) \\
            &= \frac{1}{M} \left( \sum_{m=1}^M \bm{1}_{k=l} \cdot \phi(v_m^\top a) \phi(v_m^\top b) \right) + \frac{1}{M} \left( \sum_{m=1}^M w_{mk} w_{ml} \phi'(v_m^\top a) \phi'(v_m^\top b) a^\top b \right) \\
            &\le c_\phi^2 + \frac{1}{M} \left( \sum_{m=1}^M |w_{mk} w_{ml}| \cdot c_{\phi'}^2 \right) \\
            &\le c_\phi^2 + \frac{c_{\phi'}^2}{M} \cdot \left( \sum_{m=1}^M w_{mk}^2 + w_{ml}^2 \right) \\
            &\le c_\phi^2 + c_{\phi'}^2 \kappa_1
        \end{split}
    \end{align}
    where the final line uses Part 1 of this proof, which bounds $\frac{1}{M} \sum_{m=1}^M \|w_m(t)\|^2$ uniformly in time by $\kappa_1$. Also, recall that data lies on the unit sphere, so $a^\top b \le 1$. Overall, we see that the spectral norm of each kernel matrix $\bm{K}_{k,l}$ is bounded by
    \begin{align}
        \| \bm{K}_{k,l} \|_2 \le 2N \left( c_\phi^2 + c_{\phi'}^2 \kappa_1 \right)
    \end{align}
    Plugging this in back into our evolution of the representations,
    \begin{align}
        \dt \sum_{k=1}^K \left \| f_k([\calX, \calX_+]) \right\|^2 \le \left( 8K c_\phi^2 + 8K c_{\phi'}^2 \kappa_1 \right) \left( \sum_{k=1}^K \left \| f_k([\calX, \calX_+]) \right\|^2 \right)
    \end{align}
    We are thus in the setting of Grönwall's inequality, and obtain
    \begin{align}
        \sum_{k=1}^K \left \| f_k([\calX, \calX_+]) \right\|^2 \le \left( \sum_{k=1}^K \| f_k([\calX, \calX_+])(0) \|^2 \right) \cdot \exp \left(T \left( 8K c_{\phi}^2 + 8K c_{\phi'}^2 \kappa_1 \right) \right) =: \kappa_2
    \end{align}
    for all time $t \le T$. When $\frac{\| \theta(0)\|^2}{M} \le \kappa_1 - 1$ this also gives a bound on $\| f_k([\calX, \calX_+])(0) \|^2$.

    \textbf{Part 3.} We now control $\| \dot{\theta}(t) \|^2$. Using the time derivatives for $w_{mk}$ and $v_{mr}$ derived in Part 1,
    \begin{align}
        \begin{split}
            |\dot{w_{mk}}(t)|^2 &= \frac{1}{N^2} \trace \left( (C-I) \left( \sum_{n=1}^N \frac{\partial}{\partial w_{mk}} \left( f_i(x_n) f_j(x_n^+) + f_j(x_n) f_i(x_n^+) \right)\right)_{i,j} \right)^2 \\
            &\le \frac{1}{N^2} \left\| \sum_{n=1}^N \left( \frac{\partial}{\partial w_{mk}} \left( f_i(x_n) f_j(x_n^+) + f_j(x_n) f_i(x_n^+) \right)\right)_{i,j} \right\|_F^2 \\
            &= \mathcal{O}\left( \frac{ \kappa_2 c_{\phi}^2}{M} \right)
        \end{split}
    \end{align}
    where we used Cauchy-Schwartz for the trace inner product and $\|C-I\|_F<1$ in the second line, and $|f|^2 \le \kappa_2$ from Part 2 in the third line. Summing up,
    \begin{align}
        \sum_{m=1}^M \| \dot{w_m}(t)\|^2 = \mathcal{O}\left( \kappa_2 c_{\phi}^2 \right)
    \end{align}
    Similarly, we find that
    \begin{align}
        \begin{split}
            |\dot{v_{mr}}(t)|^2 &= \frac{1}{N^2} \trace \left( (C-I) \left( \sum_{n=1}^N \frac{\partial}{\partial v_{mr}} \left( f_i(x_n) f_j(x_n^+) + f_j(x_n) f_i(x_n^+) \right)\right)_{i,j} \right)^2 \\
            &\le \frac{1}{N^2} \left\| \sum_{n=1}^N \left( \frac{\partial}{\partial v_{mr}} \left( f_i(x_n) f_j(x_n^+) + f_j(x_n) f_i(x_n^+) \right)\right)_{i,j} \right\|_F^2 \\
            &= \mathcal{O}\left( \frac{\kappa_2 \|w_m\|^2 c_{\phi'}^2}{M}\right)
        \end{split}
    \end{align}
    Summing up and exploiting part 1 to bound $\frac{1}{M} \sum_{m=1}^M \|w_m(t)\|^2$ uniformly in time by $\kappa_1$, we find that
    \begin{align}
        \sum_{m=1}^M \| \dot{v_m} \|^2 = \mathcal{O}(\kappa_1 \kappa_2 c_{\phi'}^2)
    \end{align}
    Combining both results, we see that there exists $\kappa>0$ such that $ \sup_{t \le T} \| \dot{\theta}(t) \| \le \kappa$ with probability $\ge 1- \epsilon$ over random initialization.
\end{proof}

\subsection{Proof of Lemma \ref{lemma_multidim_1}}\label{app:lemma_multidim_1}

\begin{proof}
    Recall that for any $x$ and any $k \in [K]$, the $k$-th output $f_k$ evolves over time as follows.
    \begin{align}
    \begin{split}
        \dt f_k(x) &= \langle \frac{\partial f_k(x)}{\partial \theta}, \dt \theta \rangle \\
        &= - \sum_{m=1}^M \frac{\partial f_k(x)}{\partial \theta_m} \frac{\partial \loss}{\partial \theta_m} \\
        &= - \sum_{m=1}^M \frac{\partial f_k(x)}{\partial \theta_m} \sum_{n=1}^N \sum_{l=1}^K \frac{\partial \loss}{\partial f_l(x_n)} \frac{\partial f_l(x_n)}{\partial \theta_m} + \frac{\partial \loss}{\partial f_l(x_n^+)} \frac{\partial f_l(x_n^+)}{\partial \theta_m} \\
        &= - \sum_{n=1}^N \sum_{l=1}^K \frac{\partial \loss}{\partial f_l(x_n)} \left( K_\theta(x, x_n) \right)_{k,l} + \frac{\partial \loss}{\partial f_l(x_n^+)} \left( K_\theta(x, x_n^+\right)_{k,l}
    \end{split}
    \end{align}
    From here, we write down the time evolution of the loss. We denote $*,+$ for the symbols that refer to an anchor sample $x=x^*$, or an augmentation $x^+$.
    \begin{align}\label{eq:lossode_multi_bt}
        \begin{split}
            \dt \loss(t) &= \sum_{i=1}^N \frac{\partial \loss}{\partial f(x_i)} \dt f(x_i) + \frac{\partial \loss}{\partial f(x_i^+)} \dt f(x_i^+) \\
            &= -\sum_{i,n=1}^N \sum_{k,l=1}^K \sum_{\alpha, \beta=*, +} \frac{\partial \loss}{\partial f_k(x_i^\alpha)} K_{k,l}(x_i^\alpha, x_n^\beta) \frac{\partial \loss}{\partial f_k(x_i^\beta)} \\
            &= - \sum_{i,n=1}^N \left( \frac{\partial \loss}{\partial f(x_i)} \right)^\top K(x_i, x_n) \left( \frac{\partial \loss}{\partial f(x_n)} \right) + \\
            &\qquad \qquad \left( \frac{\partial \loss}{\partial f(x_i^+)} \right)^\top K(x_i^+, x_n) \left( \frac{\partial \loss}{\partial f(x_n)} \right) + \\
            &\qquad \qquad \left( \frac{\partial \loss}{\partial f(x_i)} \right)^\top K(x_i, x_n^+) \left( \frac{\partial \loss}{\partial f(x_n^+)} \right) + \\
            &\qquad \qquad \left( \frac{\partial \loss}{\partial f(x_i^+)} \right)^\top K(x_i^+, x_n^+) \left( \frac{\partial \loss}{\partial f(x_n^+)} \right)
        \end{split}
    \end{align}
    this can be expressed as 
    \begin{align}
        \dt \loss(t) = - u(t)^\top \bm K(t) u(t)
    \end{align}
    where we have defined
    \begin{align}
        u(t) &=
        \begin{pmatrix}
            \frac{\partial \loss}{\partial f_1(x_1)} & \dots & \frac{\partial \loss}{\partial f_K(x_1)} & \dots & \frac{\partial \loss}{\partial f_K(x_N)} & \frac{\partial \loss}{\partial f_1(x_1^+)} & \dots & \frac{\partial \loss}{\partial f_K(x_N^+)}
        \end{pmatrix}^\top \\
        \bm K(t) &= 
        \begin{pmatrix}
         K(x_1, x_1) & \dots & K(x_1, x_N) & \dots & K(x_1, x_N^+) \\
         \dots & \dots & \dots & \dots & \dots \\
         K(x_N, x_1) & \dots & K(x_N, x_N) & \dots & K(x_N, x_N^+) \\
         K(x_1^+, x_1) & \dots & K(x_1^+, x_N) & \dots & K(x_1^+, x_N^+) \\
         \dots & \dots & \dots & \dots & \dots \\
         K(x_N^+, x_1) & \dots & K(x_N^+, x_N) & \dots & K(x_N^+, x_N^+) \\
        \end{pmatrix}
    \end{align}
\end{proof}

\subsection{Proof of Theorem \ref{theo:multiv_lossbound}}\label{app:theo:multiv_lossbound}

\begin{proof}
    Observe that
    \begin{align}
        \dt \loss(t) = - u(t)^\top \bm K(t) u(t) \le - \|u(t)\|^2 \lambda_{min}(\bm K(t))
    \end{align}
    with $\bm K$ and $u(t)$ as defined earlier. Recall from Appendix \ref{app:theo:multi_gradientbound} (part 3 of the proof) that
    \begin{align}
    \begin{split}
        \frac{\partial \loss}{\partial f_k(x_i)} = \frac{2}{N} e_k ^\top (C-I) f(x_i^+)
    \end{split}
    \end{align}
    Thus, we can lower bound the squared Euclidean norm of $u(t)$ as
    \begin{align}
    \begin{split}
        \|u(t)\|^2 &= \sum_{k=1}^K \sum_{i=1}^N \left( \frac{\partial \loss}{\partial f_k(x_i)} \right)^2 + \left( \frac{\partial \loss}{\partial f_k(x_i^+)} \right)^2 \\
        &= \frac{4}{N^2} \sum_{k=1}^K \sum_{i=1}^N f(x_i^+)^\top (C-I) e_k e_k^\top (C-I) f(x_i^+) + f(x_i)^\top (C-I) e_k e_k^\top (C-I) f(x_i) \\
        &= \frac{4}{N^2} \sum_{i=1}^N f(x_i^+)^\top (C-I)^2 f(x_i^+) + f(x_i)^\top (C-I)^2 f(x_i) \\
        &\ge \frac{4}{N^2} \sum_{i=1}^N f(x_i^+)^\top (C-I)^2 f(x_i) + f(x_i)^\top (C-I)^2 f(x_i^+) \\
        &= \frac{4}{N^2} \trace \left((C-I)^2 \left( \sum_{i=1}^N f(x_i) f(x_i^+)^\top + f(x_i^+) f(x_i)^\top \right) \right) \\
        &= \frac{8}{N} \trace( (C-I)^2 C)
    \end{split}
    \end{align}
    The inequality used above relies on the following fact: For any positive semi-definite matrix $A = Q^\top \Lambda Q \in \R^K$ (where $Q$ is orthonormal and $\Lambda$ is diagonal with non-negative entries), and for any $v,w \in \R^K$, Cauchy-Schwartz yields
    \begin{align}
    \begin{split}
        v^\top A w + w^\top A v &= (Qv)^\top \Lambda (Qw) + (Qw)^\top \Lambda (Qv) \\
        &= 2 \sum_{k=1}^K \lambda_k (Qv)_k (Qw)_k \\
        &\le \sum_{k=1}^K \lambda_k \left( (Qv)_k^2 + (Qw)_k^2 \right) \\
        &= v^\top Q^\top \Lambda Q v + w^\top Q^\top \Lambda Q w \\
        &= v^\top A v + w^\top A w
    \end{split}
    \end{align}
    All that remains to be done is to show that
    \begin{align}
    \| u(t) \|^2 \ge \kappa \|C-I\|_{F}^2
    \end{align}
    for some $\kappa$ that is time-independent. By Von Neumann's trace inequality we can certainly choose $\kappa = \lambda_{min}(C(t))$ at any given $t$. Thus, as long as the smallest eigenvalue of $C(t)$ is lower bounded until convergence, we are fine. Recall that we assume $\loss(0) \le 1-\rho < 1$. Now assume that at some time $t'$ we have $\lambda_{min}(C(t')) < 1-\sqrt{\loss(0)}$. Then, denoting $\lambda_1, \dots, \lambda_K$ for the eigenvalues of $C(t')$, it follows that
    \begin{align}\label{eq:loss_contradiction}
    \loss(t') = \trace \left((C-I)^2 \right) = \sum_{k=1}^K (\lambda_k-1)^2 \ge (\lambda_{min}(C(t'))-1)^2 > \loss(0)
    \end{align}
    This is a contradiction to the non-increasing nature of the loss under gradient flow. Hence, as long as $\loss(0) < 1$, we can lower bound $\lambda_{min}(C)$ uniformly in time by $1 - \sqrt{\loss(0)}$ from below. Thus, it holds that
    \begin{align}
        \| u(t) \|^2 > \frac{8}{N} (1 - \sqrt{\loss(0)}) \cdot \loss(t)
    \end{align}
    and hence
    \begin{align}
        \dt \loss(t) < -\frac{8(1 - \sqrt{\loss(0)}) \cdot \lambda_{min}(\bm K(t))}{N} \cdot \loss(t) \le - \frac{8(1-\sqrt{1-\rho}) \cdot  \lambda_{min}(\bm K(t))}{N} \cdot \loss(t)
    \end{align}
    holds for all $t$. We now take $M$ to be so large that the entries of the kernel matrix change by less than some $\gamma >0$ for all $t \le T$, where $\gamma >0$ is small enough to ensure that $\lambda_{min}( \bm K(t)) \ge \lambda/2$ for all $t \le T$. Such $M$ certainly exists, because Theorem \ref{theo1} guarantees that the change in the entries of the kernel matrix $\bm K(t)$ are no more than $\mathcal{O}(\frac{R^2}{\sqrt{M}})$ until time $T$, where $R = \kappa T$ following the discussion preceding Theorem \ref{theo:multi_gradientbound}. Overall, this implies that whenever $M \ge M_1$ for some $M_1 \in \N$, we have $\dt \loss(t) < - \eta \loss(t)$ for all $t \le T$.
\end{proof}

\subsection{Ensuring $\loss(0)<1$ for ReLU}\label{app:scaling_v}

\begin{lemma}\label{lemma:scaling_v} \textbf{(Scaling the first layer ensures small loss)}
    Consider the ReLU activation function. For any dataset $\calX, \calX_+$, there exist $s>0$ and $M_0 \in \N$ such that for all $M \ge M_0$ and under $w_{mk} \sim \mathcal{N}(0,1), v_{mr} \sim \mathcal{N}(0, s^2)$, the loss at initialization satisfies $\loss(0) < 1$ with high probability.
\end{lemma}

\begin{proof}
    In expectation, we have the following expression for the cross-moments at initialization.
    \begin{align}
    \begin{split}
        \E[C_{kl}] &= \frac{1}{2N} \sum_{n=1}^N \E[ f_k(x_n) f_l(x_n^+) ] + \E[ f_k(x_n^+) f_l(x_n) ] \\
        &= \frac{1}{2N} \sum_{n=1}^N \frac{1}{M} \sum_{m,m'=1}^M \E \left[ w_{mk} \left( \phi(v_m^\top x_n) \phi(v_{m'}^\top x_n^+) + \phi(v_m^\top x_n^+) \phi(v_{m'}^\top x_n) \right) w_{m'l} \right]
    \end{split}
    \end{align}
    For $k \neq l$, the fact that all $w_{mk}$ are independent and zero mean shows that $\E[C_{kl}] = 0$. When $k=l$, we have $\E[w_{mk} w_{m'l}] = \bm 1(m=m')$, since the weights $w_{mk}$ are standard Gaussians. Therefore,
    \begin{align}
    \begin{split}
        \E[C_{kk}] &= \frac{1}{N} \sum_{n=1}^N \frac{1}{M} \sum_{m}^M \E \left[\phi(v_m^\top x_n) \phi(v_{m}^\top x_n^+)\right]
    \end{split}
    \end{align}
    For ReLU, the expectation is strictly positive and scales as $s^2$. Thus, there certainly exists $s>0$ such that 
    \begin{align}
    \E[C_{kk}] = \frac{2K-1}{2K}
    \end{align}
    For large $M$, the variance of each $C_{kl}$ is $\mathcal{O}(1/M)$ because the weights are independent. Thus, the matrix $I-C$ concentrates around $(2K)^{-1} I_K$, so $\loss(0) = \|C(0) - I\|_F^2 < 1$ with high probability.
\end{proof}

\clearpage

\section{PROOFS FROM SECTION \ref{sec:warmup}}\label{app:warmup}

\subsection{Proof of Theorem \ref{theo:univ_gradientbound}}\label{app:theo:univ_gradientbound}

\begin{proof}
    Recall that $T>0$ is a fixed, width-independent point in time. We show that there exists $\kappa>0$ such that $\sup_{t \le T} \| \dot{\theta} \| \le \kappa$ for all networks of width $M>8T$, with high probability over randomly initialized weights. Specifically, for any $\epsilon>0$, we know that there exists $\kappa_1>0$ such that
    \begin{align}
        \frac{\| \theta(0)\|^2}{M} \le \kappa_1 - 1
    \end{align}
    with probability at least $1-\epsilon$, because the weights are independent Gaussians at initialization. We condition everything on this event. In this proof, we also assume $\loss(0)<1$ under said event. We need this condition anyway for Theorem \ref{theo:multiv_lossbound} and therefore include it into our bounds here as well to simplify matters. However, any width-independent bound on $\loss(0)$ is sufficient. Denote $c_\phi = \max_{t \in \R} | \phi(t)|$ and $c_{\phi'} = \max_{t \in \R} \left| \dt \phi(t) \right|$. We will show the following three statements.
    \begin{enumerate}
        \item Firstly, we show that 
        \begin{align}
            \dt \left( \sum_{m=1}^M w_m^2(t) \right) \le 8
        \end{align}
        for all $t \le T$. This immediately implies that, with probability $\ge 1- \epsilon$ over random initialization,
        \begin{align}
            \sup_{t \le T} \frac{1}{M} \left( \sum_{m=1}^M w_m^2(t) \right) \le \kappa_1
        \end{align}
        because for any $t$, we can write
        \begin{align}
            \frac{1}{M} \left( \sum_{m=1}^M w_m^2(t) \right) \le \frac{1}{M} \left( \sum_{m=1}^M w_m ^2(0) \right) + \frac{8T}{M} \le \frac{1}{M} \left( \sum_{m=1}^M w_m^2(0) \right) + 1
        \end{align}
        where we used $M > 8T$ and the fact that $\frac{\| \theta(0)\|^2}{M} \le \kappa_1 - 1$ with high probability.
        \item From there, we bound the maximum squared value that any representation takes until time $T$, that is
        \begin{align}
            |f|^2 = \max_{n \in [N]} \sup_{t \le T} \max \left( |f(x_n)|^2, |f(x_n^+)|^2 \right)
        \end{align}
        by some $\kappa_2 > 0$ (again independent of $M$).
        \item We combine both statements to show that there exists $\kappa>0$ such that $\sup_{t \le T} \| \dot{\theta} \| \le \kappa$ with probability at least $1-\epsilon$.
    \end{enumerate}
    
    \textbf{Part 1.} 
    We begin by computing
    \begin{align} \begin{split}
        \frac{\partial \loss}{\partial w_m} = &2(C-1) \left( \frac{1}{N} \sum_{n=1}^N \left[f(x_n) \repam \phi(v_m^\top x_n^+) + f(x_n^+) \repam \phi(v_m^\top x_n)\right] \right)
    \end{split} \end{align}
    which equates to $- \dt w_m$ under gradient flow, and
    \begin{align} \begin{split}
        \frac{\partial \loss}{\partial v_{mj}} = &2(C-1) \left( \frac{1}{N} \sum_{n=1}^N \left[f(x_n) \repam w_m \phi'(v_m^\top x_n^+) (x_n^+)^{(j)} + f(x_n^+) \repam w_m \phi'(v_m^\top x_n) (x_n)^{(j)}\right] \right)
    \end{split} \end{align}
    which equates to $- \dt v_{mj}$ under gradient flow.
    Clearly,
    \begin{align} \begin{split}
        \dt \left( w_m^2 \right) = & 4(1-C) \left( \frac{1}{N} \sum_{n=1}^N \left[ f(x_n) f_m(x_n^+) + f(x_n^+) f_m(x_n) \right]\right)
    \end{split} \end{align}
    where we write $f_m(x) = \repam w_m \phi(v_m^\top x)$. Noticing that $f(x) = \sum_{m=1}^M f_m(x)$, we arrive at
    \begin{align} \begin{split}
        \dt \left( \sum_{m=1}^M w_m^2 \right) &= 4(1-C) \left( \frac{1}{N} \sum_{n=1}^N \left[ f(x_n) \left(\sum_{m=1}^M f_m(x_n^+)\right) + f(x_n^+) \left(\sum_{m=1}^M f_m(x_n)\right) \right] \right) \\
        &= 8(1-C) C
    \end{split} \end{align}
    The loss $\loss(t) = (C-1)^2$ is non-increasing under gradient flow. Thus, since $\loss(0)<1$, we have $C \in (0,1)$. We obtain $\dt \left( \sum_{m=1}^M w_m^2(t) \right) \le 8$ for all $t \le T$ as desired.

    \textbf{Part 2.} We consider the evolution of the representations $f(x)$ over time, where $x \in \calX, \calX_+$ Note that for any $x$,
    \begin{align} \begin{split}
    \dt f(x) &= \langle \dtheta f(x), \dt \theta \rangle \\
    &= - \sum_{m=1}^M \frac{\partial f(x)}{\partial \theta_m} \frac{\partial \loss}{\partial \theta_m} \\
    &= - \sum_{m=1}^M \frac{\partial}{\partial \theta_m} f(x) \left[ \sum_{n=1}^N \left( \frac{\partial \loss}{\partial f(x_n)} \frac{\partial f(x_n)}{\partial \theta_m} + \frac{\partial \loss}{\partial f(x_n^+)} \frac{\partial f(x_n^+)}{\partial \theta_m} \right)\right] \\
    &= - \sum_{n=1}^N \left[ K_{\theta}(x, x_n) \frac{\partial \loss}{\partial f(x_n)} + K_{\theta}(x, x_n^+) \frac{\partial \loss}{\partial f(x_n^+)} \right]
    \end{split} \end{align}
    In our setting
    \begin{align} \begin{split}
        &\frac{\partial \loss}{\partial f(x_n)} = 2(C -1)\left(\frac{1}{N} f(x_n^+) \right) \\
        &\frac{\partial \loss}{\partial f(x_n^+)} = 2(C -1)\left(\frac{1}{N} f(x_n) \right)
    \end{split} \end{align}
    Hence, for any $i \in [N]$, we have
    \begin{align} \begin{split}
        \dt f(x_i) &= \frac{2(1-C)}{N} \cdot \sum_{n=1}^N \left[ K_{\theta}(x_i, x_n) f(x_n^+) + K_{\theta}(x_i, x_n^+) f(x_n) \right] \\
        &= \frac{2(1-C)}{N} \cdot K_{\theta}(x_i, [\calX,\calX_+]) \fcalxallinv
    \end{split} \end{align}
    and therefore we can write
    \begin{align}
        \dt \fcalxall = \frac{2(1-C)}{N} \cdot K_{\theta}([\calX,\calX_+], [\calX,\calX_+]) \fcalxallinv
    \end{align}
    Thus,
    \begin{align}
        \dt \left \| \fcalxall \right \|^2 &= 2 \fcalxall^\top \left( \dt \fcalxall \right) \\
        &= \frac{4(1-C)}{N} \fcalxall^\top K_{\theta}([\calX,\calX_+], [\calX,\calX_+]) \fcalxallinv
    \end{align}
    By Cauchy-Schwartz inequality, for any positive semi-definite matrix $A$, it holds that
    \begin{align}
        |y^\top A z|^2 \le (y^\top A y) \cdot (z^\top A z) \le \|y\|^2 \|z\|^2 \|A\|_2^2
    \end{align}
    In our case, this implies that
    \begin{align}
        \fcalxall^\top K_{\theta}([\calX,\calX_+], [\calX,\calX_+]) \fcalxallinv \le \left \| \fcalxall \right \|^2 \cdot \left\| K_{\theta}([\calX,\calX_+], [\calX,\calX_+]) \right\|_2
    \end{align}
    where we used the fact that
    \begin{align}
        \left \| \fcalxall \right\| = \left \| \fcalxallinv \right\|
    \end{align}
    since both vectors are just permutations of one another. It remains to show that the spectral norm of the time-dependent kernel matrix remains bounded until time $T$. Note that for any pair $(a,b) \in [\calX, \calX_+] \times [\calX, \calX_+]$, and for any time $t \le T$, we have $|a^\top b| \le 1$. Therefore, every entry of the kernel matrix is bounded via
    \begin{align}
    \begin{split}
        K_\theta(a,b) &= \left( \frac{\partial f(a)}{\partial \theta} \right)^\top \left( \frac{\partial f(b)}{\partial \theta} \right) \\
        &= \sum_{m=1}^M \left( \frac{\partial f(a)}{\partial w_m} \frac{\partial f(b)}{\partial w_m} + \sum_{j=1}^d \frac{\partial f(a)}{\partial v_{mj}} \frac{\partial f(b)}{\partial v_{mj}} \right) \\
        &= \frac{1}{M} \sum_{m=1}^M \phi(v_m^\top a) \phi(v_m^\top b) + w_m^2 \phi'(v_m^\top a) \phi'(v_m^\top b) a^\top b \\
        &\le \frac{1}{M} \sum_{m=1}^M c_{\phi}^2 + w_m^2 c_{\phi'}^2 \\
        &\le c_{\phi}^2 + c_{\phi'}^2 \left( \frac{1}{M} \sum_{m=1}^M w_m^2(t) \right) \\
        &\le c_{\phi}^2 + c_{\phi'}^2 \kappa_1
        \end{split}
    \end{align}
    where we used Part 1 in the final inequality. Since the spectral norm is upper bounded by the trace, we obtain 
    \begin{align}
        \left\| K_{\theta}([\calX,\calX_+], [\calX,\calX_+]) \right\|_2 \le 2N \left( c_{\phi}^2 + c_{\phi'}^2 \kappa_1 \right)
    \end{align}
    Consequently, going back to the evolution of the representations, we see that
    \begin{align}
        \dt \left \| \fcalxall \right \|^2 \le 8(1-C) \left( c_{\phi}^2 + c_{\phi'}^2 \kappa_1 \right) \left \| \fcalxall \right \|^2 \le 8 \left( c_{\phi}^2 + c_{\phi'}^2 \kappa_1 \right) \left \| \fcalxall \right \|^2
    \end{align}
    Grönwall's inequality now ensures that there exists $\kappa_2>0$ such that
    \begin{align}
        \left \| \fcalxall (t) \right \|^2 \le \left \| \fcalxall (0) \right \|^2 \exp \left( 8T \left( c_{\phi}^2 + c_{\phi'}^2 \kappa_1 \right) \right) =: \kappa_2
    \end{align}
    for all time $t \le T$, where we used the fact that $\left \| \fcalxall (0) \right \|^2$ can be bounded in terms of $\frac{\| \theta(0)\|^2}{M}$. Since
    \begin{align}
        |f|^2 = \max_{n \in [N]} \sup_{t \le T} \max \left( |f(x_n)|^2, |f(x_n^+)|^2 \right) \le \sup_{t \le T} \left \| \fcalxall (t) \right \|^2
    \end{align}
    this concludes part 2, giving us a bound $\kappa_2$ on $|f|^2$.
    
    \textbf{Part 3.} 
    We finally control $\| \dot{\theta}(t) \|^2$. Using the time derivatives for $w_m, v_{mj}$ we derived in Part 1,
    \begin{align}
        \begin{split}
        |\dot{w_m}(t)|^2 &= 4(1-C)^2 \left( \frac{1}{N} \sum_{n=1}^N \left[ f(x_n) \repam \phi(v_m^\top x_n^+) + f(x_n^+) \repam \phi(v_m^\top x_n)\right] \right)^2 \\
        &\le 4(1-C)^2 \left( 2|f| c_{\phi} \repam \right)^2 \\
        &= 16(1-C)^2 \frac{|f|^2 c_{\phi}^2 }{M} \\
        &\le \frac{16 \kappa_2 c_{\phi}^2}{M}
        \end{split}
    \end{align}
    for all $m \in [M]$. Similarly, it holds that
    \begin{align}
    \begin{split}
        |\dot{v_{mj}}(t)|^2 &= 4(1-C)^2 \left( \frac{1}{N} \sum_{n=1}^N \left[ f(x_n) \repam w_m \phi'(v_m^\top x_n^+) (x_n^+)^{(j)} + f(x_n^+) \repam w_m \phi'(v_m^\top x_n) (x_n)^{(j)} \right] \right)^2 \\
        &\le 4(1-C)^2 \left( \frac{2 |f| |w_m| c_{\phi'}}{\sqrt{M}}\right)^2 \\
        &\le  \frac{16 \kappa_2 w_m^2 c_{\phi'}^2}{M}
    \end{split}
    \end{align}
    for all $m \in [M], j \in [d]$. Thus, we obtain
    \begin{align}
    \begin{split}
        \| \dot{\theta}(t) \|^2 &=
        \sum_{m=1}^M \sum_{j=1}^d (\dot{w_m}(t))^2 + (\dot{v_{mj}}(t))^2 \\
        &\le \frac{1}{M} \left( \sum_{m=1}^M 16 \kappa_2 c_{\phi}^2 + 16 d \kappa_2 w_m^2 c_{\phi'}^2 \right)
    \end{split}
    \end{align}
    From part 1, $\frac{1}{M} \sum_{m=1}^M w_m^2(t) \le \kappa_1$ holds
    for all $t \le T$. Thus, defining
    \begin{align}
        \kappa = 4\sqrt{\kappa_2 c_{\phi}^2 + d \kappa_1 \kappa_2 c_{\phi'}^2} 
    \end{align}
    we obtain the desired high-probability bound on $\sup_{t \le T} \|\dot{\theta}(t)\|$ and this conclude the proof.
\end{proof}

\subsection{Proof of Lemma \ref{lemma:loss_dynamics_1dim}}\label{app:lemma_loss_dynamics_1dim}

\begin{proof}
    First recall what we showed in Appendix \ref{app:theo:univ_gradientbound}. The function representations $f(x)$ evolve as
    \begin{align} \begin{split}
        \dt f(x) = - \sum_{n=1}^N K_{\theta}(x, x_n) \frac{\partial \loss}{\partial f(x_n)} + K_{\theta}(x, x_n^+) \frac{\partial \loss}{\partial f(x_n^+)}
    \end{split} \end{align}
    for any $i \in [N]$.
    From there, it is apparent that 
    \begin{align} \begin{split}
    \dt \loss(t) &= \sum_{i=1}^N \frac{\partial \loss}{\partial f(x_i)} \dt f(x_i) + \frac{\partial \loss}{\partial f(x_i^+)} \dt f(x_i^+) \\
    &= - \sum_{i,n=1}^N \frac{\partial \loss}{\partial f(x_i)} \frac{\partial \loss}{\partial f(x_n)} K(x_n, x_i) + \frac{\partial \loss}{\partial f(x_i^+)} \frac{\partial \loss}{\partial f(x_n)} K(x_n, x_i^+) + \\
    &\qquad \qquad \frac{\partial \loss}{\partial f(x_i)} \frac{\partial \loss}{\partial f(x_n^+)} K(x_n^+, x_i) + \frac{\partial \loss}{\partial f(x_i^+)} \frac{\partial \loss}{\partial f(x_n^+)} K(x_n^+, x_i^+) \\
    &= - \frac{4(C-1)^2}{N^2} \sum_{i,n=1}^N f(x_i^+) f(x_n^+) K(x_n, x_i) + f(x_i) f(x_n^+) K(x_n, x_i^+) + \\
    &\qquad \qquad \qquad \qquad \quad
    f(x_i^+) f(x_n) K(x_n^+, x_i) + f(x_i) f(x_n) K(x_n^+, x_i^+) \\
    &= - \frac{4}{N^2} \cdot \loss(t) \cdot \fcalxall^\top \bm{K}(t) \fcalxall
    \end{split} \end{align}
    where the kernel matrix $\bm K(t)$ is defined as
    \begin{align}
        \bm{K}(t) = 
        \begin{pmatrix}
        K_{\theta}(\calX_+, \calX_+) & K_{\theta}(\calX_+, \calX) \\
        K_{\theta}(\calX, \calX_+) & K_{\theta}(\calX, \calX) 
        \end{pmatrix}
    \end{align}
\end{proof}

\subsection{Proof of Theorem \ref{theo:univ_lossbound}}\label{app:theo:univ_lossbound}

\begin{proof}
    Note that the entries of the kernel matrix $\bm K$ change by no more than $\mathcal{O}(\frac{\kappa^2 T^2}{\sqrt{M}})$ up to time $T$. This is a consequence of Theorem \ref{theo1} which states that the change is $\mathcal{O}(\frac{R^2}{\sqrt{M}})$, and Theorem \ref{theo:univ_gradientbound} which asserts that $R = \mathcal{O}(\kappa T)$. Since the spectral norm is upper bounded by the trace of a matrix, for any $\gamma>0$ there exists some large $M$ such that
    \begin{align} \begin{split}\label{eq:Kspectralbound}
    \| \bm K(t) - \bm K(0) \|_2 \le \trace \left( \bm K(t) - \bm K(0) \right) \le \gamma
    \end{split} \end{align}
    for all $t \le T$. We pick $\gamma = \frac{\lambda}{2}$, which ensures that
    \begin{align} \begin{split}
        \lambda_{min}( \bm K(t)) \ge \frac{\lambda}{2}
    \end{split} \end{align}
    Define
    \begin{align} \begin{split}\label{eq:zdef}
        z(t) = \fcalxall(t)
    \end{split} \end{align}
    We claim that $\frac{\|z(t)\|^2}{2N} \ge C(0)$ for all time $t \ge 0$. Indeed, assume this was not the case at a certain time $t' > 0$, where instead
    \begin{align} \begin{split}\label{eq:absurd}
        \frac{\|z(t')\|^2}{2N} < C(0)
    \end{split} \end{align}
    Then,
    \begin{align} \begin{split}
        C(t') &= \frac{1}{2N} \sum_{n=1}^N f(x_n)f(x_n^+) + f(x_n^+)f(x_n) \\
        &= \frac{1}{2N} \left( \fcalxall(t') \right)^\top \left( \fcalxallinv(t') \right) \\
        &\le \frac{1}{2N} \|z(t')\|^2 \\
        &< C(0)
    \end{split} \end{align}
    The first inequality is Cauchy-Schwartz, using the fact that
    \begin{align} \begin{split}
        \left\| \fcalxall (t') \right\| = \left\| \fcalxallinv (t') \right\| = \|z(t')\| 
    \end{split} \end{align}
    The second inequality plugs in our assumption \eqref{eq:absurd}. But if $C(t') < C(0)$, then this leads to a contradiction. Since $C(0) \in (0,1)$, we would obtain
    \begin{align} \begin{split}
        \loss(t') = (1-C(t'))^2 > (1-C(0))^2 = \loss(0)
    \end{split} \end{align}
    This is impossible, because gradient flow monotonically decreases $\loss$. Thus, for all $t \le T$, we know that
    \begin{align} \begin{split}
        \frac{\|z(t)\|^2}{2N} \ge C(0)
    \end{split} \end{align}
    From there, we complete the proof.
    \begin{align} \begin{split}
        g(t) &= \frac{4}{N^2} \cdot z(t)^\top \bm K(t) z(t) \\
        &\ge \frac{4}{N^2} \cdot \| z(t) \|^2 \cdot \lambda_{min}( \bm K(t)) \\
        &> \frac{2 \lambda}{N^2} \cdot \|z(t)\|^2 \\
        &\ge \frac{2 \lambda}{N^2} \cdot \left( 2N C(0) \right) \\
        &= \frac{4\lambda C(0)}{N} \\
        &= \frac{4\lambda (1 - \sqrt{\loss(0)})}{N} \\
        &\ge \frac{4\lambda (1 - \sqrt{1-\rho})}{N} \\
        &= \eta
    \end{split} \end{align}
    as desired.
\end{proof}

\clearpage

\section{EXTENDING THEOREM \ref{theo:univ_gradientbound} TO RELU ACTIVATIONS}\label{app:relu}

In this section, we prove that the weights remain in a ball of bounded radius even for ReLU activations, until any fixed time $T>0$. The proof is not entirely rigorous, because gradient flow w.r.t. the weights is ill-defined due to the non-differentiability of the ReLU function. We restrict ourselves to an analysis of the one-dimensional setting.

\begin{proof}
    Similar to the case for smooth bounded activations with bounded first derivatives (Appendix \ref{app:theo:univ_gradientbound}), our proof consists of three parts. Remember that $T>0$ is a fixed, width-independent point in time. We will show the following three statements, and assume $\loss(0) < 1$ throughout. However, as before, all that we really need is boundedness of the loss at initialization. Note that our arguments are slightly reshuffled compared to the proof of Theorem \ref{theo:univ_gradientbound}. This is due to unboundedness of the ReLU function.
    \begin{enumerate}
        \item Firstly, we show that $\dt \| \theta(t) \|^2 \le 16$ for all $t \le T$. 
        \item Secondly, we verify that for any $t \le T$, it holds that
        \begin{align}
            \| \dot{\theta}(t) \|^2 \le \frac{16 |f|^2 d \|\theta\|^2}{M}
        \end{align}
        where we denote
        \begin{align}
            |f|^2 = \max_{n \in [N]} \sup_{t \le T} \max \left( |f(x_n)|^2, |f(x_n^+)|^2 \right)
        \end{align}
        for the maximum squared value that any representation takes until time $T$.
        \item Thirdly, we show that there exists a constant $\kappa_1$ (again depending only on $T$) such that $|f|^2 \le \kappa_1$.
    \end{enumerate}
    
    Together, this will be enough. Part 1 implies that for all $t \le T$, we have 
    \begin{align}
        \frac{\|\theta (t)\|^2}{M} \le \frac{\|\theta(0)\|^2 + 16T}{M} \le \frac{\|\theta(0)\|^2}{M} + 2
    \end{align}
    since $M > 8T$. Then, combining part 2 and part 3, we arrive at
    \begin{align}
        \| \dot{\theta}(t) \|^2 \le \frac{16 d \kappa_1 (\|\theta(0)\|^2+16T)}{M} \le 16 d \kappa_1 \left( \frac{\|\theta(0)\|^2}{M} + 2 \right)
    \end{align}
    for all $t \le T$. Note that for any $\epsilon>0$ we can choose $\kappa$ such that this expression is smaller than $\kappa$ with probability at least $1-\epsilon$. This is possible because the weights at initialization are independent Gaussians.
    
    \textbf{Part 1.} 
    We begin by computing
    \begin{align} \begin{split}
        \frac{\partial \loss}{\partial w_m} = &2(C-1) \left( \frac{1}{N} \sum_{n=1}^N \left[f(x_n) \repam \phi(v_m^\top x_n^+) + f(x_n^+) \repam \phi(v_m^\top x_n)\right] \right)
    \end{split} \end{align}
    which equates to $- \dt w_m$ under gradient flow, and
    \begin{align} \begin{split}
        \frac{\partial \loss}{\partial v_{mj}} = &2(C-1) \left( \frac{1}{N} \sum_{n=1}^N \left[f(x_n) \repam w_m \phi'(v_m^\top x_n^+) (x_n^+)^{(j)} + f(x_n^+) \repam w_m \phi'(v_m^\top x_n) (x_n)^{(j)}\right] \right)
    \end{split} \end{align}
    which equates to $- \dt v_{mj}$ under gradient flow.
    Thus, writing $f_m(x) = \repam w_m \phi(v_m^\top x)$, it holds that
    \begin{align} \begin{split}
        \dt \left( w_m^2 \right) = \dt \left( \|v_m\|^2 \right) = & 4(1-C) \left( \frac{1}{N} \sum_{n=1}^N \left[ f(x_n) f_m(x_n^+) + f(x_n^+) f_m(x_n) \right]\right)
    \end{split} \end{align}
    where we used properties of the ReLU function $\phi$, namely that
    \begin{align} \begin{split}
        \phi(v_m^\top x) = \phi'(v_m^\top x) \cdot v_m^\top x
    \end{split} \end{align}
    Hence, noticing that $f(x) = \sum_{m=1}^M f_m(x)$, we arrive at
    \begin{align} \begin{split}
        \dt \left( \|\theta\|^2 \right) &= \dt \left( \sum_{m=1}^M \left(w_m^2 + \|v_m\|^2\right) \right) \\
        &= 8(1-C) \left( \frac{1}{N} \sum_{n=1}^N \left[ f(x_n) \left(\sum_{m=1}^M f_m(x_n^+)\right) + f(x_n^+) \left(\sum_{m=1}^M f_m(x_n)\right) \right] \right) \\
        &= 16(1-C) C
    \end{split} \end{align}
    The loss is decreasing under gradient flow. Due to $\loss(0)<1$, we have $C \in (0,1)$. We obtain $\dt \left( \|\theta\|^2 \right) \le 16$ as claimed.
    
    \textbf{Part 2.} 
    We now control $\| \dot{\theta}(t) \|$. First note that $| \phi(v_m^\top x) | \le \|v_m\|$ due to $\|x\| = 1$ for all $x \in (\calX, \calX_+)$. Moreover,
    \begin{align}
        \begin{split}
        |\dot{w_m}(t)|^2 &= 4(1-C)^2 \left( \frac{1}{N} \sum_{n=1}^N \left[ f(x_n) \repam \phi(v_m^\top x_n^+) + f(x_n^+) \repam \phi(v_m^\top x_n)\right] \right)^2 \\
        &\le 4(1-C)^2 \left( \frac{\|v_m\|}{N \sqrt{M}} \sum_{n=1}^N \left[ |f(x_n)| + |f(x_n^+)| \right] \right)^2 \\
        &\le 4(1-C)^2 \left( \frac{2\|v_m\||f|}{\sqrt{M}} \right)^2 \\
        &= \frac{16(1-C)^2 \|v_m\|^2 |f|^2}{M} \\
        &\le \frac{16 \|v_m\|^2 |f|^2}{M}
        \end{split}
    \end{align}
    holds for all $m \in [M]$. Similarly, it holds that
    \begin{align}
        |\dot{v_{mj}}(t)|^2 \le  \frac{16 w_m^2 |f|^2}{M}
    \end{align}
    for all $m \in [M], j \in [d]$. Thus, we obtain
    \begin{align}
    \begin{split}
        \| \dot{\theta}(t) \|^2 &=
        \sum_{m=1}^M \left[|\dot{w_m}(t)|^2 + \sum_{j=1}^d |\dot{v_{mj}}(t)|^2 \right] \\
        &\le \frac{16 |f|^2 }{M} \left[ \sum_{m=1}^M \left( \|v_m\|^2 + d w_m^2 \right) \right]\\
        &\le \frac{16 |f|^2 d \|\theta\|^2}{M}
    \end{split}
    \end{align}
    as desired.
    
    \textbf{Part 3.} Just as in Appendix \ref{app:theo:univ_gradientbound}, we look at the evolution of the representations over time. We obtain
    \begin{align}
        \dt \fcalxall = \frac{2(1-C)}{N} \cdot K_{\theta}([\calX,\calX_+], [\calX,\calX_+]) \fcalxallinv
    \end{align}
    and thus
    \begin{align}
    \begin{split}
        \dt \left \| \fcalxall \right \|^2 &= 2 \fcalxall^\top \left( \dt \fcalxall \right) \\
        &= \frac{4(1-C)}{N} \fcalxall^\top K_{\theta}([\calX,\calX_+], [\calX,\calX_+]) \fcalxallinv
    \end{split}
    \end{align}
    By Cauchy-Schwartz inequality,
    \begin{align}
        \fcalxall^\top K_{\theta}([\calX,\calX_+], [\calX,\calX_+]) \fcalxallinv \le \left \| \fcalxall \right \|^2 \cdot \left\| K_{\theta}([\calX,\calX_+], [\calX,\calX_+]) \right\|_2
    \end{align}
    Again, we must bound the spectral norm of the time-dependent kernel matrix until time $T$. This is now slightly different for ReLU. Note that for any pair $(a,b) \in [\calX, \calX_+] \times [\calX, \calX_+]$, we have
    \begin{align}
    \begin{split}
        K_\theta(a,b) &= \left( \frac{\partial f(a)}{\partial \theta} \right)^\top \left( \frac{\partial f(b)}{\partial \theta} \right) \\
        &= \sum_{m=1}^M \frac{\partial f(a)}{\partial w_m} \frac{\partial f(b)}{\partial w_m} + \sum_{j=1}^d \frac{\partial f(a)}{\partial v_{mj}} \frac{\partial f(b)}{\partial v_{mj}} \\
        &= \frac{1}{M} \sum_{m=1}^M \phi(v_m^\top a) \phi(v_m^\top b) + w_m^2 \phi'(v_m^\top a) \phi'(v_m^\top b) a^\top b \\
        &\le \frac{1}{M} \sum_{m=1}^M \|v_m\|^2 + w_m^2 \\
        &= \frac{\|\theta (t)\|^2}{M} \\
        &\le \frac{\|\theta(0)\|^2}{M} +1
        \end{split}
    \end{align}
    using part 1. Since the spectral norm is upper bounded by the trace, we obtain 
    \begin{align}
        \left\| K_{\theta}([\calX,\calX_+], [\calX,\calX_+]) \right\|_2 \le N \left( \frac{\|\theta(0)\|^2}{M} +1 \right)
    \end{align}
    which is $\mathcal{O}(1)$. Consequently, going back to the evolution of the representations, we see that
    \begin{align}
        \dt \left \| \fcalxall \right \|^2 \le 4 \left( \frac{\|\theta(0)\|^2}{M} +1 \right) \left \| \fcalxall \right \|^2
    \end{align}
    Grönwall's inequality now ensures that
    \begin{align}
        \left \| \fcalxall (t) \right \|^2 \le \left \| \fcalxall (0) \right \|^2 \exp \left( 4T \left( \frac{\|\theta(0)\|^2}{M} +1 \right) \right)=: \kappa_1
    \end{align}
    for all time $t \le T$. Since 
    \begin{align}
        |f|^2 = \max_{n \in [N]} \sup_{t \le T} \max \left( |f(x_n)|^2, |f(x_n^+)|^2 \right) \le \sup_{t \le T} \left \| \fcalxall (t) \right \|^2
    \end{align}
    this concludes part 3 and finishes the proof.
\end{proof}

\clearpage

\section{PROOFS FROM SECTION \ref{sec:implications}}\label{app:implications}

\subsection{Proof of Theorem \ref{theo:learningtheory}}\label{app:theo:learningtheory}

\begin{proof}  
    Our proof is based on previous works on the theoretical analysis of Kernel PCA \citep{blanchard2007statistical}.
    Denote by $\bar{L}(W) = \| W \bar{ \Gamma} W^* - I \|_F$ the square root of the population Barlow Twins loss. Denote by $L(W) = \| W \Gamma W^* - I \|_F$ the empirical version, computed with respect to the empirical cross-moment matrix
    \begin{align}
        \Gamma = \frac{1}{2N} \sum_{n=1}^N z_n (z_n^+)^* + z_n^+ z_n^*
    \end{align}
    Given independently drawn positive pairs $(z_1, z_1^+), \dots, (z_N, z_N^+)$, define the map
    \begin{align}
        \psi(z_1, \dots, z_N) = \sup_{ \|W\| \le B } \left| \bar{L}(W) - L(W) \right|
    \end{align}
    where $W: \mathcal{H} \rightarrow \R^K$ is a bounded linear operator. The function $\psi$ satisfies a bounded differences inequality, because for any $(y_n, y_n^+)$ (and with $\Gamma'$ denoting the cross-moment matrix w.r.t. the new sample) it holds that
    \begin{align}
    \begin{split}
        \big| \psi(z_1, \dots, z_n, \dots, z_N) - \psi(z_1, \dots, y_n, \dots, z _N) \big| &\le \sup_{\|W\| \le B} \left| |\bar{L}(W) - \|W \Gamma W^* - I\|_F | - |\bar{L}(W) - \|W \Gamma' W^* - I\|_F | \right|\\
        &\le \sup_{ \|W\| \le B} \big| \|W \Gamma W^* - I \|_F - \|W \Gamma' W^* - I\|_F \big| \\
        &\le \sup_{ \|W\| \le B}  \| W ( \Gamma - \Gamma') W^* \|_F \\
        &\le B^2 \left\| \frac{1}{2N} \left( z_n (z_n^+)^* + (z_n^+) z_n^* - y_n (y_n^+)^* - (y_n^+) y_n^* \right) \right\|_{\text{HS}} \\
        &\le \frac{B^2}{N} \left( \|z_n (z_n^+)^*\|_{\text{HS}} + \|y_n (y_n^+)^*\|_{\text{HS}} \right) \\
        &\le \frac{2B^2 S^2}{N}
    \end{split}
    \end{align}
    We bounded the difference of suprema by the supremum of the difference, used the reverse triangle inequality twice, and exploited $\|AB\|_{\text{HS}} \le \|A\| \cdot\|B\|_{\text{HS}}$ for operators $A,B$. In the final step, we use
    \begin{align}
        \|z (z^+)^*\|_{\text{HS}} = \sqrt{\trace \left( z (z^+)^* (z^+) z^* \right)} =  \sqrt{\|z \|^2 \|z^+\|^2} \le S^2
    \end{align}
    for any $z,z^*$ (recall that they lie in a ball of radius no more than $S$ in $\mathcal{H}$).
    Overall, McDiarmid's inequality yields that for any $\epsilon > 0$
    \begin{align}
        \Prob \left( \psi(z_1, \dots, z_n) - \E_{z_1, \dots, z_N} \left[\psi(z_1, \dots, z_n) \right] > \epsilon \right) \le \exp \left( - \frac{2\epsilon^2}{N \cdot (\frac{2B^2S^2}{N})^2} \right) = \exp \left( - \frac{N\epsilon^2}{2B^4 S^4} \right)
    \end{align}
    The expectation of $\psi$ (w.r.t. samples $z_1, \dots, z_N$) can be bounded as follows.
    \begin{align}
    \begin{split}
        \E_{z_1, \dots, z_N} \left[\psi(z_1, \dots, z_n) \right]
        &= \E_{z_1, \dots, z_N} \left[ \sup_{ \| W \| \le B } \left| \bar{L}(W) - L(W) \right| \right] \\
        &\le \E_{z_1, \dots, z_N} \left[ \sup_{ \| W \| \le B }  \left| \|W \bar{ \Gamma} W^* - I \|_F - \|W \Gamma W^* - I \|_F \right| \right] \\
        &\le B^2 \E_{z_1, \dots, z_N} \left[ \| \Gamma - \bar{ \Gamma} \|_{\text{HS}} \right]
    \end{split}
    \end{align}
    where we used the reverse triangle inequality again.
    Let us write
    \begin{align}
        \Gamma_i = \frac{1}{2} \left( z_i (z_i^+)^* + z_i^+ z_i^* \right)
    \end{align}
    so that $\Gamma = \frac{1}{N} \sum_{i=1}^N \Gamma_i$. Continuing with Jensen's inequality,
    \begin{align}
    \begin{split}
        B^2 \E_{z_1, \dots, z_N} \left[ \|\Gamma - \bar{ \Gamma} \|_{\text{HS}} \right] &\le B^2 \left( \E_{z_1, \dots, z_N} \left[ \left \| \Gamma - \bar{ \Gamma} \right \|^2_{\text{HS}} \right] \right)^{1/2} \\
        &= B^2 \left( \E_{z_1, \dots, z_N} \left[ \left \|  \frac{1}{N} \sum_{i=1}^N \left( \Gamma_i - \bar{ \Gamma} \right) \right \|^2_{\text{HS}} \right] \right)^{1/2} \\
        &= B^2 \left( \E_{z_1, \dots, z_N} \left[ \frac{1}{N^2} \sum_{i,j=1}^N \trace \left( (\Gamma_i - \bar{ \Gamma})^* (\Gamma_j - \bar{ \Gamma}) \right) \right] \right)^{1/2} \\
        &= B^2 \left( \frac{1}{N^2} \sum_{i=1}^N \E_{\Gamma_i} \left[ \left\| \Gamma_i - \bar{ \Gamma} \right\|_{\text{HS}}^2 \right] \right)^{1/2}  \\
        &= B^2 \left( \frac{1}{N} \cdot \E_{\Gamma_i} \left[ \left\| \Gamma_i - \bar{ \Gamma} \right\|_{\text{HS}}^2 \right] \right)^{1/2} \\
        &= \frac{B^2}{\sqrt{N}} \left( \E_{\Gamma_i} \left[ \left\| \Gamma_i - \bar{ \Gamma} \right\|_{\text{HS}}^2 \right] \right)^{1/2}
    \end{split} 
    \end{align}
    In the third line, we rewrote the Hilbert-Schmidt norm in terms of the trace inner product, then used independence and zero mean of all $\Gamma_i - \bar{ \Gamma}, \Gamma_j - \bar{ \Gamma}$ to drop the traces of $i \neq j$, and finally wrote everything in terms of a single expectation.
    \begin{align}
        \bm V = \E \left[ \left\| \Gamma_i - \bar{ \Gamma} \right\|_{\text{HS}}^2 \right]
    \end{align}
    is the variance of the random operator $\Gamma_i$ w.r.t. the Hilbert-Schmidt norm. An unbiased estimator of $\bm V$ is given by the empirical variance
    \begin{align}
        \bm {\hat V} = \frac{1}{N'(N'-1)} \sum_{\substack{i < j \\ i,j \in [N']}} \left \| \Gamma_i - \Gamma_j \right \|_{\text{HS}}^2 = \frac{1}{2N' (N'-1)} \sum_{i \neq j} \left \| \Gamma_i - \Gamma_j \right \|_{\text{HS}}^2
    \end{align}
    This is a function of independent random operators $\Gamma_1, \dots \Gamma_{N'}$ that again satisfies a bounded differences inequality. For any collection of $\Gamma_1, \dots, \Gamma_n$ and any ``new'' element $\Gamma'_n$, it holds that
    \begin{align}
    \begin{split}
        &\left| \bm {\hat V}(\Gamma_1, \dots, \Gamma_n, \dots, \Gamma_{N'}) - \bm {\hat V}(\Gamma_1, \dots, \Gamma'_n, \dots, \Gamma_{N'}) \right| \\
        &\le \frac{1}{2N'(N'-1)} \sum_{i=1}^{N'} \left| \left \| \Gamma_i - \Gamma_n \right \|_{\text{HS}}^2 - \left \| \Gamma_i - \Gamma'_n \right \|_{\text{HS}}^2 \right| \\
        &\le \frac{1}{2N'(N'-1)} \sum_{i=1}^{N'} 4 \| \Gamma_i \|_{\text{HS}}^2 + 2 \| \Gamma_n \|_{\text{HS}}^2 + 2\| \Gamma_n' \|_{\text{HS}}^2 \\
        &= \frac{4S^4}{N'-1}
    \end{split}
    \end{align}
    where we again used the fact that data is contained in a ball of radius $S$, and hence $\| \Gamma_i \|^2_{\text{HS}} \le S^4 $ for all $i$.
    Using McDiarmid's inequality once again, we conclude that with probability $\ge 1 - \epsilon$ over random training data,
    \begin{align}
        \bm V \le \bm{ \hat V} + \exp \left( - \frac{(N'-1)^2 \epsilon^2}{8S^8 N'} \right)
    \end{align}
    Overall, this shows that with probability $\ge 1 - \epsilon$, the expected value of $\psi(z_1, \dots, z_N)$ is bounded as
    \begin{align}
        \E \left[\psi(z_1, \dots, z_n) \right] &\le \frac{B^2}{\sqrt{N}} \left( \bm{ \hat V} + \exp \left( - \frac{(N'-1)^2 \epsilon^2}{8S^8 N'} \right) \right)^{1/2}
    \end{align}
    By a union bound, we conclude that with probability $\ge 1 - 2\epsilon$ over training samples
    \begin{align}\label{eq:nu_N_eps}
        \left| \bar{L}(W) - L(W) \right| \le \frac{B^2}{\sqrt{N}} \left( \bm{ \hat V} + \exp \left( - \frac{(N'-1)^2 \epsilon^2}{8S^8 N'} \right) \right)^{1/2} + \exp \left( - \frac{N\epsilon^2}{2B^4 S^4} \right) =: \nu(N, \epsilon)
    \end{align}
    holds uniformly over all $W$ with $\|W\| \le B$. Using the fact that $\loss(W) \le \delta$ almost surely over randomly drawn samples, we push $L(W) \le \sqrt{\delta}$ to the right. Then, we square both sides and use $(a+b+c)^2 \le 3(a^2 + b^2 + c^2)$. This gives
    \begin{align}\label{eq:nu_N_eps_delta}
        \bar{\loss}(W) &\le 3 \delta + \frac{3B^4}{N} \left( \bm{ \hat V} + \exp \left( - \frac{(N'-1)^2 \epsilon^2}{8 S^8 N'} \right) \right) + 3\exp \left( - \frac{N\epsilon^2}{B^4 S^4} \right) =: \nu(N, \epsilon, \delta)
    \end{align}
    with probability at least $1 - 2 \epsilon$. This concludes the proof.
\end{proof}

\subsection{Proof of Lemma \ref{lemma:hatv_kernel}}\label{app:lemma:hatv_kernel}

\begin{proof}
    For all $i \neq j$, we can expand the Hilbert-Schmidt norm and use the cyclic property of the trace to obtain
    \begin{align}
        \begin{split}
            \| \Gamma_i - \Gamma_j \|_{\text{HS}}^2 &= \| \Gamma_i \|_{\text{HS}}^2 + \| \Gamma_j \|_{\text{HS}}^2 - 2 \trace \left( \Gamma_i^* \Gamma_j \right) \\
            &= \frac{1}{2} \left( \|\psi(x_i) \|^2 \|\psi(x_i^+)\|^2 + \langle \psi(x_i), \psi(x_i^+) \rangle^2 + \| \psi(x_j) \|^2 \| \psi(x_j^+) \|^2 + \langle \psi(x_j), \psi(x_j^+) \rangle^2 \right) \\
            &\quad - \left( \langle \psi(x_i) , \psi(x_j) \rangle \cdot \langle \psi(x_i^+) , \psi(x_j^+) \rangle + \langle \psi(x_i), \psi(x_j^+) \rangle \cdot \langle \psi(x_i^+) , \psi(x_j) \rangle \right)
        \end{split}
    \end{align}
    which can be written purely in terms of the kernel, giving
    \begin{align}
        \begin{split}
            \| \Gamma_i - \Gamma_j \|_{\text{HS}}^2 = &0.5 \left( K(x_i,x_i) K(x_i^+, x_i^+) + K(x_i, x_i^+)^2 \right) + \\
            &0.5 \left( K(x_j,x_j) K(x_j^+, x_j^+) + K(x_j, x_j^+)^2 \right) - \\
            &K(x_i, x_j) K(x_i^+, x_j^+) - K(x_i, x_j^+) K(x_i^+, x_j)
        \end{split}
    \end{align}
    which concludes the proof.
\end{proof}

\subsection{Proof of Corollary \ref{cor:nn_lossbound}}\label{app:nn_lossbound}

\begin{proof}
    Denote $f$ for the neural network and $g = W \psi(x)$ for the corresponding NTK model from Equation \eqref{eq:ntk_model}. We begin by bounding the difference in the loss in terms of the difference between $f$ and $g$. The reverse triangle inequality gives
    \begin{align}
    \begin{split}
        \left| \bar{\loss}(f)^{1/2} - \bar{\loss}(g)^{1/2} \right| &=
            \Big| \left\| C(f) - I \right\|_F - \left \| C(g) - I \right \|_F \Big| \\
            &\le \left\| C(f) - C(g) \right\|_F \\
            &= \left\| \E \left[ \frac{1}{2} \left( f(x) f(x^+)^\top + f(x^+) f(x)^\top \right) - \frac{1}{2} \left( g(x) g(x^+)^\top - g(x^+) g(x)^\top \right) \right] \right\|_F \\
            &= 0.5 \left( \sum_{i,j=1}^K \E \left[ f_i(x) f_j(x^+) + f_i(x^+) f_j(x) - g_i(x) g_j(x^+) - g_i(x^+) g_j(x) \right]^2 \right)^{1/2}
    \end{split}
    \end{align}
    Adding and subtracting terms of the form $f_i(x) - g_i(x)$, we see that this expression is upper bounded by
    \begin{align}
        \begin{split}
            &0.5 \left( \sum_{i,j=1}^K \E \left[ \left|f_i(x) f_j(x^+) + f_i(x^+) f_j(x) - g_i(x) g_j(x^+) - g_i(x^+) g_j(x) \right| \right]^2 \right)^{1/2} \le  \\ &0.5 \left( \sum_{i,j=1}^K \zeta^2 \E \left[|f_i(x)| + |g_j(x^+)| + |f_j(x^+)| + |g_j(x)| \right]^2 \right)^{1/2} \le \\ &0.5 \left( \sum_{i,j=1}^K \zeta^2 \E \left[|g_i(x)| + |g_j(x^+)| + |g_j(x^+)| + |g_j(x)| + 2 \zeta \right]^2 
            \right)^{1/2} \le \\ 
            &0.5 \sqrt{ \sum_{i,j=1}^K \zeta^2 \left( 4BS + 2 \zeta \right)^2 } = \\ 
            &K \zeta (2BS + \zeta)
        \end{split}
    \end{align}
    where we used the fact that $|f_i(x) - g_i(x)| \le \zeta$ for all $i \in [K]$ and that $g_i(x) = \langle w_i, \phi(x) \rangle \le \|w_i\| \| \psi(x)\| \le B S$ since the data is contained in a ball of radius $S$ in the Hilbert space, and $\|W\| \le B$. The same derivation works for the (square root of the) empirical loss, so
    \begin{align}
        \left| \loss(f)^{1/2} - \loss(g)^{1/2} \right| \le K \zeta (2BS + \zeta)
    \end{align}
    From there, we bound $\bar{\loss}(f)^{1/2}$ as follows.
    \begin{align}
        \begin{split}
            \bar{\loss}(f)^{1/2} &\le \bar{\loss}(g)^{1/2} + K \zeta (2BS + \zeta) \\
            &\le {\loss}(g)^{1/2} + \nu(N,\epsilon) + K \zeta (2BS + \zeta) \qquad\text{with probability $\ge 1-2\epsilon$.} \\
            &\le {\loss}(f)^{1/2} + \nu(N,\epsilon) + 2K \zeta (2BS + \zeta)  \\
            &\le \sqrt{\delta} + \nu(N,\epsilon) + 2K \zeta (2BS + \zeta)
        \end{split}
    \end{align}
    Here, we plugged in the slack term $\nu(N,\epsilon)$ from Theorem \ref{theo:learningtheory}, see Equation \eqref{eq:nu_N_eps}. This leads to high-probability bounds in the second step. Moreover, we exploited $\loss(f) \le \delta$ in the final step. Squaring both sides, plugging in
    \begin{align}
        \left( \sqrt{\delta} + \nu(N,\epsilon) \right)^2 \le \nu(N,\epsilon, \delta)
    \end{align}
    from Equation \eqref{eq:nu_N_eps_delta}, and using $(a+b)^2 \le 2(a^2 + b^2)$ we obtain
    \begin{align}
        \bar{\loss}(f) \le 2 \nu(N,\epsilon,\delta) + 8K^2 \zeta^2 (2BS + \zeta)^2 
    \end{align}
    with probability at least $1-2\epsilon$. Since the approximation of $f$ through $g$ up to $\zeta$ holds with probability at least $1-\epsilon$ we get the desired statement.
\end{proof}

\clearpage

\section{DYNAMICS OF LINEAR BARLOW TWINS --- CONVERGENCE FROM SMALL INITIALIZATION}\label{app:dynamics}

We write the empirical linearized Barlow Twins loss as
\begin{align}\label{eq:linear_bt}
    \loss(W) = \left \| W \Gamma W^\top - I_K \right \|_F^2
\end{align}
where $\Gamma$ is the empirical cross-moment matrix between the anchors and the augmentations, either in $\R^d$ or mapped into some RKHS. 

It is known from recent works \citep{simon2023stepwise} that $\frac{\partial \loss}{\partial W} = 4(W \Gamma W^\top - I) W \Gamma$ and therefore, under gradient flow, it holds that $\dt W = 4(I - W \Gamma W^\top) W \Gamma$. We decompose the matrix $W$ into $W = Q + W_0$, where the rows of $W_0$ are in the nullspace of $\Gamma$, and the rows of $Q$ are in its orthogonal complement. Since $W \Gamma = Q \Gamma$, we have $\loss(W) = \loss(Q)$, and $\dt W_0 = 0$. Therefore, the part of $W$ that starts in the nullspace $\Gamma$ stays completely unchanged.
We may therefore w.l.o.g. restrict ourselves to the setting where $\Gamma$ has only nonzero eigenvalues. 
We obtain
\begin{align}
    \dt C = 4(I-C) W \Gamma^2 W^\top + 4 W \Gamma^2 W^\top (I-C)
\end{align}
For an eigendecomposition $\Gamma = U D U^\top$ with diagonal $D$ and orthogonal $U$, we write $|\Gamma| = U |D| U^\top$ where $|D|_{ii} = |D_{ii}|$. Moreover, denote by $\mu_\Gamma > 0$ the smallest nonzero eigenvalue of $\Gamma$ in absolute value. Then,
\begin{align}
\begin{split}
    \lambda_{min}(W \Gamma^2 W^\top) &= \min_{\|u\|=1} u^\top W |\Gamma|^{1/2} |\Gamma| |\Gamma|^{1/2} W^\top u \\
    &\ge \mu_{\Gamma} \| |\Gamma|^{1/2} W^\top u \|^2 \\
    &\ge \mu_\Gamma \lambda_{min}( W |\Gamma| W^\top ) \\
    &\ge \mu_\Gamma \lambda_{min}(C)
\end{split}
\end{align}
Therefore, for any $x$ having unit norm, the quadratic form $x^\top C x$ satisfies
\begin{align}\label{eq:c_evolution_linear}
    \begin{split}
        &\dt \left( x^\top C x \right) \ge 8\mu_\Gamma \lambda_{min}(C) (1 - x^\top C x), \\
        &\dt \left( x^\top C x \right) \le 8\lambda_{max}(\Gamma) \lambda_{max}(C) (1-x^\top C x)
    \end{split}
\end{align}
where we used Von-Neumann's trace inequality on the product of the two matrices $I-C$ and $W \Gamma^2 W^\top$.
This shows that as long as all eigenvalues of $C(0)$ are contained in $(0,1)$, the map $t \mapsto x^\top C(t) x$ is strictly increasing for all $x$, and has an equilibrium at $1$. Consequently, we see that the smallest eigenvalue of $C$ is strictly increasing, and that no eigenvalue of $C$ can ever exceed $1$. In addition,
\begin{align}\begin{split}
    \dt \loss(t) &= 2 \trace \left( (C-I) \frac{\partial C}{\partial t} \right) \\ 
    &= -16 \trace \left( (C-I)^2 W \Gamma^2 W^\top \right) \\
    &\le -16 \lambda_{min}(C) \mu_\Gamma \cdot \loss(t) \\
    &\le -16 \lambda_{min}(C(0)) \mu_\Gamma \cdot \loss(t) \\
    &= - \eta \loss(t)
\end{split} \end{align}
where $\eta = 16 \lambda_{min}(C(0)) \mu_{\Gamma}$ . Therefore, $\loss(t) \le \loss(0) \exp \left( -\eta t  \right)$ and we see that after $T = \frac{-\log \delta}{\eta}$ time, the loss is smaller than $\delta$.

\clearpage

\clearpage

\section{TECHNICAL RESULTS}

\subsection{Grönwall's Inequality}

\begin{lemma}\label{app:grönwall}\textbf{(Grönwall's inequality)}
    Let $u(t)$ and $\beta(t)$ be two continuous functions on an interval $[a,b]$. Suppose $u(t)$ is differentiable in $(a,b)$ and satisfies
    \begin{align}
        \dt u(t) \le \beta(t) u(t)
    \end{align}
    for all $t \in (a,b)$. Then,
    \begin{align}
        u(t) \le u(0) \exp \left( \int_{a}^t \beta(s) ds \right)
    \end{align}
    for all $t \in [a,b]$.
\end{lemma}

\subsection{Bounding $c_0$ from Theorem \ref{theo1}}\label{app:boundc0}

Recall that the activation function $\phi$ and its derivative $\phi'$ are bounded by $c_\phi$ and $c_{\phi'}$ respectively, and that we assume $\theta \in B(\theta_0,R)$, and that inputs have unit norm. Thus, for all $k \in [k]$,
\begin{align}
    \begin{split}
        \| \nabla_\theta f_k(a; \theta \|^2 &= \sum_{m} \left( \frac{\partial f_k(a;\theta)}{\partial w_{m,k}} \right)^2 + \sum_{m,j} \left( \frac{\partial f_k(a;\theta)}{\partial v_{m,j}} \right)^2 \\
        &= \frac{1}{M} \sum_{m} \phi(v_m^\top a)^2 + \frac{1}{M} \sum_{m,j} w_{m,k}^2 \left( \phi'(v_m^T a) a^{(j)} \right)^2 \\
        &\le c_\phi^2 + \frac{c_{\phi'}^2}{M} \sum_{m} w_m^2 \|a\|^2 \\
        &\le c_\phi^2 + \frac{c_{\phi'}^2 \| \theta\|^2}{M} \\
        &\le c_\phi^2 + \frac{c_{\phi'}^2 \| \theta - \theta_0 \|^2}{M} + \frac{c_{\phi'}^2 \| \theta_0 \|^2}{M} \\
        &\le c_\phi^2 + \frac{c_{\phi'}^2 R^2}{M} + \frac{c_{\phi'}^2 \| \theta_0 \|^2}{M}
    \end{split}
\end{align}
Whenever $R<\sqrt{M}$, this expression is $\mathcal{O}(1)$ independent of the network width, because $\| \theta_0 \|^2 = \mathcal{O}(M)$ with high probability.

\subsection{McDiarmid's inequality}

\begin{lemma}\label{app:mcdiarmid}\textbf{(McDiarmid's inequality)}
    Let $(Z_1, \dots , Z_N) = Z$ be a finite sequence of independent random variables, each with values in $\mathcal{Z}$ and let $\phi : \mathcal{Z}^N \rightarrow \R$ be a measurable function such that $|\phi(Z)-\phi(Z')| \le \nu_n$  whenever $Z, Z'$ differ only differ $n$-th coordinate. Then, for every $\epsilon>0$, it holds that
    \begin{align}
        \mathbb{P} \left( \phi(Z) - \E[ \phi(Z)] > \epsilon \right) \le \exp \left( -\frac{2 \epsilon^2}{\|\nu\|_2^2}\right)
    \end{align}
\end{lemma}

\clearpage

\section{EXPERIMENTAL DETAILS \& ADDITIONAL EXPERIMENTS}\label{app:exp}

All reported experiments have been run 10 times, across random seeds 0-9. The mean as well as standard deviation have been reported through plots for all experiments. 
All experiments were run using the publicly available Google Colaboratory (\url{https://colab.research.google.com/drive/12weqAhMLbv5KJ8MlUi80iaEjtDn8lMcv?usp=sharing}). The experiments were run using CUDA enabled PyTorch on a T4 GPU with 15 GB memory. In this appendix, we include some additional experiments. 

\paragraph{Experiments for ReLU.} 
We first recreate the three plots provided in the main paper (NTK change till convergence, epochs till convergence and $L_2$ norm of representation difference) for the case of a single hidden layer neural network with ReLU activation:

\begin{figure}[h]
\includegraphics[width=1\columnwidth]{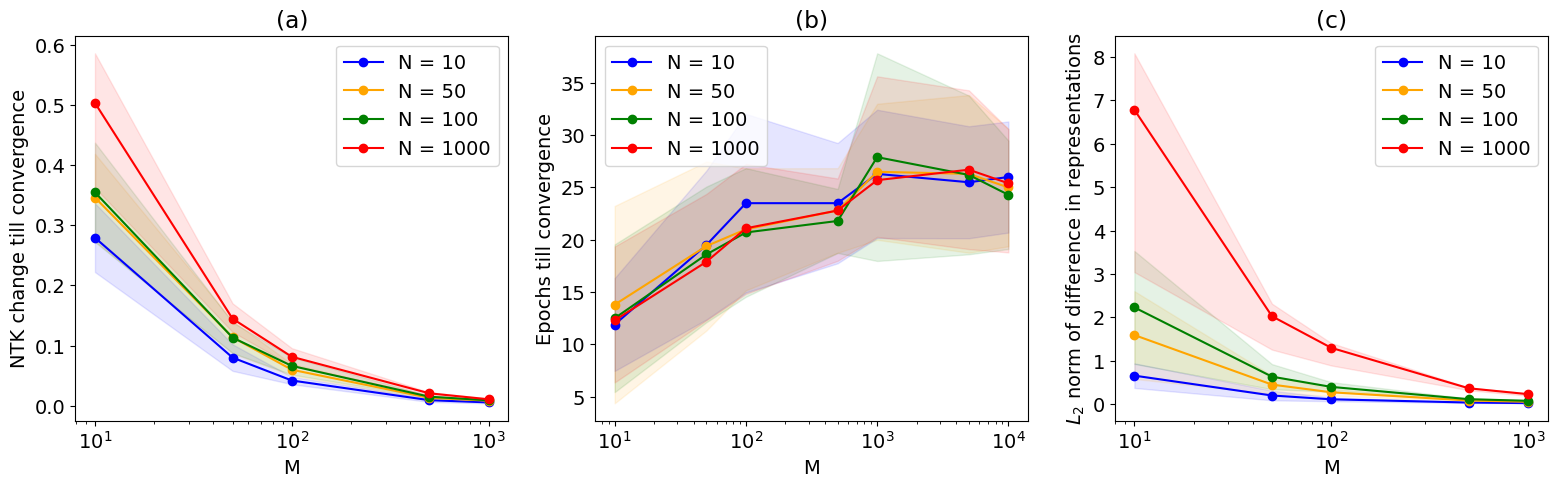}
\caption{For a fixed sample size $N$, we plot different quantities for varying network width $M$. We then vary $N$ and plot: (a) NTK change till convergence (b) Training Epochs till convergence (c) Squared norm of difference between representations of neural network and corresponding kernel model.}
\label{fig:relu}
\end{figure}

As we can see in Figure \ref{fig:relu}, all three plots closely resemble the corresponding plots with TanH activation. This empirically validates that our claims hold for ReLU activation function as well, for which our derivation was not entirely rigorous (due to the non-differentiability of ReLU at zero). 

\paragraph{Deeper ReLU networks.}
We repeat the same experiment for a 3 hidden layer neural network with ReLU activation:

\begin{figure}[h]
\includegraphics[width=1\columnwidth]{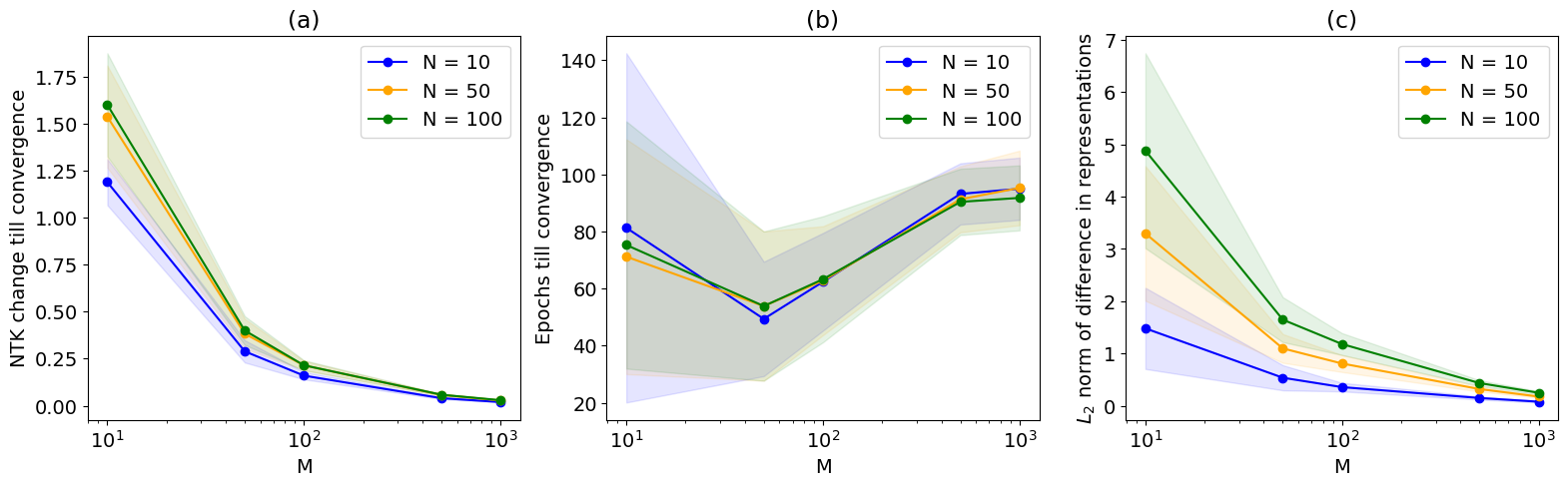}
\caption{For a fixed sample size $N$, we plot different quantities for varying network width $M$. We then vary $N$ and plot: (a) NTK change till convergence (b) Training Epochs till convergence (c) Squared norm of difference between representations of neural network and corresponding kernel model.}
\label{fig:3relu}
\end{figure}

While our proof is only for single-hidden layer neural networks, Figure \ref{fig:3relu} suggests that the analysis holds for deeper networks as well. We leave this for future work to justify theoretically.

\paragraph{Multivariate embeddings.} 
We also verify the constancy of the NTK for multivariate embeddings. Table \ref{table:ntk_change_5dim} shows the change of the NTK for $K=5$ output dimensions, for varying width $M$ and sample sizes $N$.

\begin{table}[h]\label{table:ntk_change_5dim}
    \centering
    \begin{tabular}{ccc}
        \hline
        $N$ & $M$ & Change in NTK ($\times 10^{2}$) \\
        \hline
        10  & 50   & 9.50  $\pm$ 2.26  \\
        10  & 100  & 5.50  $\pm$ 1.00  \\
        10  & 500  & 0.95  $\pm$ 0.28  \\
        10  & 1000 & 0.45  $\pm$ 0.12  \\
        50  & 50   & 17.19 $\pm$ 2.35  \\
        50  & 100  & 10.14 $\pm$ 1.17  \\
        50  & 500  & 2.26  $\pm$ 0.42  \\
        50  & 1000 & 1.06  $\pm$ 0.21  \\
        100 & 50   & 20.47 $\pm$ 1.77  \\
        100 & 100  & 11.83 $\pm$ 2.13  \\
        100 & 500  & 2.73  $\pm$ 0.46  \\
        100 & 1000 & 1.26  $\pm$ 0.25  \\
        \hline
    \end{tabular}
    \caption{Change in NTK values for different $N$ and $M$}\label{tab:ntk_changes}
\end{table}

Clearly, as $M$ grows, the change in the NTK until convergence decreases.

\paragraph{Other non-contrastive loss functions.} Our code also includes additional experiments for a theory-friendly version of the VIC-Reg loss, as well as a simplified version of BYOL. Just as for Barlow Twins, we observe near-constancy of the NTK at large width.

\end{document}